%% file: entropy_rate-nips.tex
\documentclass{article}

\usepackage[final]{nips_2018}

\usepackage[x11names]{xcolor}
\usepackage[utf8]{inputenc} 
\usepackage[T1]{fontenc}    
\usepackage{url}            
\usepackage{booktabs}       
\usepackage{amsfonts}       
\usepackage{microtype}      
\usepackage{amsmath}
\usepackage{amsthm}
\usepackage{amsfonts}
\usepackage{bm,bbm}
\usepackage{amssymb}
\usepackage{subfigure}
\usepackage{graphicx}
\usepackage{epstopdf}
\usepackage{pgfplots}
\usepackage{tikz-cd}

\newtheorem{theorem}{Theorem}

\newtheorem{lemma}{Lemma} 
\newtheorem{corollary}{Corollary}

\newtheorem{definition}{Definition}

\def \cX {\mathcal{X}}

\def \bP {\mathbb{P}}
\def \bE {\mathbb{E}}

\def \cM {\mathcal{M}}

\def \spo {\mathsf{Poi}}

\newcommand{\trel}{\tau_{\mathrm{rel}}}
\newcommand{\stepa}[1]{\overset{\rm (a)}{#1}}
\newcommand{\stepb}[1]{\overset{\rm (b)}{#1}}

\newcommand{\ceil}[1]{{\left\lceil {#1} \right \rceil}}

\usepackage{prettyref}
\newrefformat{eq}{(\ref{#1})}
\newrefformat{thm}{Theorem~\ref{#1}}
\newrefformat{sec}{Section~\ref{#1}}
\newrefformat{algo}{Algorithm~\ref{#1}}
\newrefformat{fig}{Fig.~\ref{#1}}
\newrefformat{tab}{Table~\ref{#1}}
\newrefformat{rmk}{Remark~\ref{#1}}
\newrefformat{def}{Definition~\ref{#1}}
\newrefformat{cor}{Corollary~\ref{#1}}
\newrefformat{lmm}{Lemma~\ref{#1}}
\newrefformat{prop}{Proposition~\ref{#1}}
\newrefformat{app}{Appendix~\ref{#1}}
\newrefformat{ex}{Example~\ref{#1}}

\usepackage[
            bookmarksnumbered=true,
            bookmarksopen=true,
            colorlinks=true,
            citecolor=red,
            linkcolor=blue,
            anchorcolor=red,
            urlcolor=blue]{hyperref} 

\newcommand{\Poi}{\mathsf{Poi}}
\newcommand{\pth}[1]{\left( #1 \right)}

\newcommand{\sth}[1]{\left\{ #1 \right\}}

\newcommand{\TV}{\mathsf{TV}}

\newcommand{\calR}{\mathcal{R}}
\newcommand{\calG}{\mathcal{G}}
\newcommand{\calH}{\mathcal{H}}
\newcommand{\calE}{\mathcal{E}}
\newcommand{\bfC}{\mathbf{C}}
\newcommand{\opt}{\mathsf{opt}}

\newcommand{\multi}{\mathsf{multi}}

\newcommand{\emp}{\mathsf{emp}}
\newcommand{\Hemp}{\bar{H}_{\emp}}
\newcommand{\Hopt}{\bar{H}_{\opt}}

\newcommand{\bfU}{\mathbf{U}}
\newcommand{\bfR}{\mathbf{R}}

\newcommand{\diag}{\mathsf{diag}}

\newcommand{\reals}{\mathbb{R}}

\newcommand{\Expect}{\mathbb{E}}

\newcommand{\Prob}{\mathbb{P}}
\newcommand{\prob}[1]{\Prob\left(#1\right)}

\newcommand{\numberthis}{\addtocounter{equation}{1}\tag{\theequation}}

\title{Entropy Rate Estimation for Markov Chains with Large State Space}
\author{
	Yanjun Han \\
	Department of Electrical Engineering\\
	Stanford University\\
	Stanford, CA 94305 \\
	\texttt{ yjhan@stanford.edu  } \\
	\And
	Jiantao Jiao \\
	Department of Electrical Engineering and Computer Sciences\\
	University of California, Berkeley\\	
	Berkeley, CA 94720\\
	\texttt{jiantao@berkeley.edu}\\
	\And
	Chuan-Zheng Lee, Tsachy Weissman \\
	Department of Electrical Engineering\\
	Stanford University\\
	Stanford, CA 94305 \\
	\texttt{ \{czlee, tsachy\}@stanford.edu  } \\
	\And
	Yihong Wu \\
	Department of Statistics and Data Science\\
	Yale University\\
	New Haven, CT 06511 \\
	\texttt{ yihong.wu@yale.edu } \\
	\And
	Tiancheng Yu \\
	Department of Electronic Engineering \\
	Tsinghua University \\
	Haidian, Beijing 100084 \\
	\texttt{ thueeyutc14@foxmail.com } \\	
}

\begin{document}
	
	\maketitle
	
  \input{abstract}
	
\input{main}

  \clearpage
	
	\bibliographystyle{plain}
	\bibliography{di}
	\clearpage
	
	\appendix
	\input{appendix}

\end{document}

%% file: abstract.tex
\begin{abstract}
Entropy estimation is one of the prototypical problems in distribution property testing.
To consistently estimate the Shannon entropy of a distribution on $S$ elements with independent samples, the optimal sample complexity scales sublinearly with $S$ as $\Theta(\frac{S}{\log S})$ as shown by Valiant and Valiant \cite{Valiant--Valiant2011}. Extending the theory and algorithms for entropy estimation to dependent data, this paper considers the problem of estimating the entropy rate of a stationary reversible Markov chain with $S$ states from a sample path of $n$ observations. We show that
\begin{itemize}
	\item Provided the Markov chain mixes not too slowly, \textit{i.e.}, the relaxation time is at most $O(\frac{S}{\ln^3 S})$, consistent estimation is achievable when $n \gg \frac{S^2}{\log S}$.
	\item Provided the Markov chain has some slight dependency, \textit{i.e.}, the relaxation time is at least $1+\Omega(\frac{\ln^2 S}{\sqrt{S}})$,
	consistent estimation is impossible when $n \lesssim \frac{S^2}{\log S}$.
\end{itemize}
Under both assumptions, the optimal estimation accuracy is shown to be $\Theta(\frac{S^2}{n \log S})$. In comparison, the empirical entropy rate requires at least $\Omega(S^2)$ samples to be consistent, even when the Markov chain is memoryless. In addition to synthetic experiments, we also apply the estimators that achieve the optimal sample complexity to estimate the entropy rate of the English language in the Penn Treebank and the Google One Billion Words corpora, which provides a natural benchmark for language modeling and relates it directly to the widely used perplexity measure.
\end{abstract}


%% file: main.tex
\section{Introduction}

Consider a stationary stochastic process $\{X_t\}_{t = 1}^{\infty}$, where each $X_t$ takes values in a finite alphabet $\cX$ of size $S$. The \emph{Shannon entropy rate} (or simply \emph{entropy rate}) of this process is defined as~\cite{Cover--Thomas2006}
\begin{equation}
\label{eqn.generalentropyratedef}
\bar{H}  = \lim_{n\to \infty} \frac{1}{n} H(X^n),
\end{equation}
where \[H(X^n) = \sum_{x^n \in \cX^n} P_{X^n}(x^n) \ln \frac{1}{P_{X^n}(x^n)}\] is the \emph{Shannon entropy} (or \emph{entropy}) of the random vector $X^n = (X_1,X_2,\ldots,X_n)$ and $P_{X^n}(x^n) = \prob{X_1=x_1,\ldots,X_n=x_n}$ is the joint probability mass function. Since the entropy of a random variable depends only on its distribution, we also refer to the entropy $H(P)$ of a discrete distribution $P = (p_1,p_2,\ldots,p_S)$, defined as
$
H(P) = \sum_{i = 1}^S p_i \ln \frac{1}{p_i}.
$

The Shannon entropy rate is the fundamental limit of the expected logarithmic loss when predicting the next symbol, given the all past symbols. It is also the fundamental limit of data compressing for stationary stochastic processes in terms of the average number of bits required to represent each symbol \cite{Cover--Thomas2006,Cesa--Lugosi2006}. Estimating the entropy rate of a stochastic process is a fundamental problem in information theory, statistics, and machine learning; and it has diverse applications---see, for example, \cite{Shannon1951prediction, lanctot2000estimating,song2010limits,takaguchi2011predictability,wang2012random,krumme2013predictability}.

There exists extensive literature on entropy rate estimation. It is known from data compression theory that the normalized codelength of any \emph{universal} code is a consistent estimator for the entropy rate as the number of samples approaches infinity. This observation has inspired a large variety of entropy rate estimators; see \textit{e.g.} \cite{Wyner--Ziv1989,Kontoyiannis--Algoet--Suhov--Wyner1998nonparametric,effros2002universal,Cai--Kulkarni--Verdu2004,Jiao--Permuter--Zhao--Kim--Weissman2013}.
However, most of this work has been in the asymptotic regime \cite{shields1996ergodic,ciuperca2005estimation}. Attention to \emph{non-asymptotic} analysis has only been more recent, and to date, almost only for i.i.d.\ data. There has been little work on the non-asymptotic performance of an entropy rate estimator for dependent data---that is, where the alphabet size is large (making asymptotically large datasets infeasible) and the stochastic process has memory. An understanding of this large-alphabet regime is increasingly important in modern machine learning applications, in particular, \emph{language modeling}. There have been substantial recent advances in probabilistic language models, which have been widely used in applications such as machine translation and search query completion. The entropy rate of (say) the English language represents a fundamental limit on the efficacy of a language model (measured by its \textit{perplexity}), so it is of great interest to language model researchers to obtain an accurate estimate of the entropy rate as a benchmark. However, since the alphabet size here is exceedingly large, and Google's One Billion Words corpus includes about two million unique words,\footnote{This exceeds the estimated vocabulary of the English language partly because different forms of a word count as different words in language models, and partly because of edge cases in tokenization, the automatic splitting of text into ``words''.\label{foot:english-vocab}} it is unrealistic to assume the large-sample asymptotics especially when dealing with combinations of words (bigrams, trigrams, etc). It is therefore of significant practical importance to investigate the optimal entropy rate estimator with limited sample size.

In the context of non-asymptotic analysis for i.i.d.\ samples, Paninski~\cite{Paninski2004} first showed that the Shannon entropy can be consistently estimated with $o(S)$ samples when the alphabet size $S$ approaches infinity. The seminal work of~\cite{Valiant--Valiant2011} showed that when estimating the entropy rate of an i.i.d.\ source, $n \gg \frac{S}{\log S}$ samples are necessary and sufficient for consistency.
The entropy estimators proposed in \cite{Valiant--Valiant2011} and refined in \cite{Valiant--Valiant2013estimating}, based on linear programming, have not been shown to achieve the minimax estimation rate. Another estimator proposed by the same authors~\cite{Valiant--Valiant2011power} has been shown to achieve the minimax rate in the restrictive regime of $\frac{S}{\ln S} \lesssim n \lesssim \frac{S^{1.03}}{\ln S}$.
Using the idea of best polynomial approximation, the independent work of~\cite{Wu--Yang2014minimax} and~\cite{Jiao--Venkat--Han--Weissman2015minimax} obtained estimators that achieve the minimax mean-square error $\Theta((\frac{S}{n \log S})^2 + \frac{\log^2 S}{n})$ for entropy estimation.
The intuition for the $\Theta(\frac{S}{\log S})$ sample complexity in the independent case can be interpreted as follows: as opposed to estimating the entire distribution which has $S-1$ parameters and requires $\Theta(S)$ samples, estimating the scalar functional (entropy) can be done with a logarithmic factor reduction of samples. For Markov chains which are characterized by the transition matrix consisting of $S(S-1)$ free parameters, it is reasonable to expect an $\Theta(\frac{S^2}{\log S})$ sample complexity. Indeed, we will show that this is correct provided the mixing is not too slow.

Estimating the entropy rate of a Markov chain falls in the general area of property testing and estimation with dependent data. The prior work \cite{kamath2016estimation} provided a non-asymptotic analysis of maximum-likelihood estimation of entropy rate in Markov chains and showed that it is necessary to assume certain assumptions on the mixing time for otherwise the entropy rate is impossible to estimate.
There has been some progress in related questions of estimating the mixing time from sample path \cite{hsu2015mixing,levin2016estimating}, estimating the transition matrix \cite{FOPS16}, and testing symmetric Markov chains \cite{daskalakis2017testing}. The current paper makes contribution to this growing field. In particular, the main results of this paper are highlighted as follows:
\begin{itemize}
    \item We provide a tight analysis of the sample complexity of the empirical entropy rate for Markov chains when the mixing time is not too large. This refines results in~\cite{kamath2016estimation} and shows that when mixing is not too slow, the sample complexity of the empirical entropy does not depend on the mixing time. Precisely, the bias of the empirical entropy rate vanishes uniformly over all Markov chains regardless of mixing time and reversibility as long as the number of samples grows faster than the number of parameters. It is its variance that may explode when the mixing time becomes gigantic.
    \item We obtain a characterization of the optimal sample complexity for estimating the entropy rate of a stationary reversible Markov chain in terms of the sample size, state space size, and mixing time, and partially resolve one of the open questions raised in~\cite{kamath2016estimation}. In particular, we show that when the mixing is neither too fast nor too slow, the sample complexity (up to a constant) does not depend on mixing time. In this regime, the performance of the optimal estimator with $n$ samples is essentially that of the empirical entropy rate with $n \log n$ samples.
		As opposed to the lower bound for estimating the mixing time in \cite{hsu2015mixing} obtained by applying Le Cam's method to two Markov chains which are statistically indistinguishable, the minimax lower bound in the current paper is much more involved, which, in addition to
	a series of reductions by means of simulation, relies on constructing two stationary reversible Markov chains with \emph{random} transition matrices \cite{bordenave2010spectrum}, so that the marginal distributions of the sample paths are statistically indistinguishable.
    \item We construct estimators that are efficiently computable and achieve the minimax sample complexity. The key step is to connect the entropy rate estimation problem to Shannon entropy estimation on large alphabets with i.i.d.\ samples. The analysis uses the idea of simulating Markov chains from independent samples by Billingsley \cite{billingsley1961statistical} and concentration inequalities for Markov chains.

    \item We compare the empirical performance of various estimators for entropy rate on a variety of synthetic data sets, and demonstrate the superior performances of the information-theoretically optimal estimators compared to the empirical entropy rate.

    \item We apply the information-theoretically optimal estimators to estimate the entropy rate of the Penn Treebank (PTB) and the Google One Billion Words (1BW) datasets. We show that even only with estimates using up to 4-grams, there may exist language models that achieve better perplexity than the current state-of-the-art.
\end{itemize}

The rest of the paper is organized as follows. After setting up preliminary definitions in Section~\ref{sec.preliminaries}, we summarize our main findings in Section~\ref{sec.mainresults}, with proofs sketched in Section~\ref{sec.proof_sketch}. Section~\ref{sec.languagemodels} provides empirical results on estimating the entropy rate of the Penn Treebank (PTB) and the Google One Billion Words (1BW) datasets. Detailed proofs and more experiments are deferred to the appendices. 

\section{Preliminaries}
\label{sec.preliminaries}

Consider a first-order Markov chain $X_0,X_1,X_2,\ldots$ on a finite state space $\mathcal{X}= [S]$ with transition kernel $T$. We denote the entries of $T$ as $T_{ij}$, that is, $T_{ij} = P_{X_2|X_1}(j|i)$ for $i, j \in \cX$.
Let $T_i$ denote the $i$th row of $T$, which is the conditional law of $X_2$ given $X_1=i$.
 Throughout the paper, we focus on first-order Markov chains, since any finite-order Markov chain can be converted to a first-order one by extending the state space~\cite{billingsley1961statistical}.

We say that a Markov chain is \emph{stationary} if the distribution of $X_1$, denoted by $\pi\triangleq P_{X_1}$, satisfies $\pi T=\pi$. 
We say that a Markov chain is \emph{reversible} if
it  satisfies the detailed balance equations,
$
\pi_i T_{ij}  = \pi_j T_{ji} 
$ for all $i,j\in\cX$. 
If a Markov chain is reversible, the (left) spectrum of its transition matrix $T$ contains $S$ real eigenvalues, which we denote as $1 = \lambda_1 \geq \lambda_2 \geq \cdots \geq \lambda_S \geq -1$. We define the \emph{spectral gap} and the \emph{absolute spectral gap} of $T$ as
$
\gamma(T) = 1 - \lambda_2$ and $\gamma^*(T) = 1 - \max_{i\geq 2} |\lambda_i|
$, respectively, 
and the \emph{relaxation time} of a reversible Markov chain as
\begin{equation}
\tau_{\mathrm{rel}}(T) = \frac{1}{\gamma^*(T)}.
\end{equation}

The relaxation time of a reversible Markov chain (approximately) captures its mixing time, which roughly speaking is the smallest $n$ for which the marginal distribution of $X_n$ is close to the Markov chain's stationary distribution. We refer to~\cite{montenegro2006mathematical} for a survey.

We consider the following observation model. We observe a sample path of a stationary finite-state Markov chain $X_0,X_1,\ldots,X_n$, whose Shannon entropy rate $\bar{H}$ in~(\ref{eqn.generalentropyratedef}) reduces to
\begin{align}
\bar{H}
= & ~ \sum_{i = 1}^S \pi_i \sum_{j = 1}^S T_{ij} \ln \frac{1}{T_{ij}} = H(X_1,X_2) - H(X_1) \label{eq:entropyratedef2}
\end{align}
where $\pi$ is the stationary distribution of this Markov chain. Let $\mathcal{M}_2(S)$ be the set of transition matrices of all stationary Markov chains on a state space of size $S$. Let $\mathcal{M}_{2,\text{rev}}(S)$ be the set of transition matrices of all stationary \emph{reversible} Markov chains on a state space of size $S$. We define the following class of stationary Markov reversible chains whose relaxation time is at most $\frac{1}{\gamma^*}$:
\begin{equation}
\mathcal{M}_{2,\text{rev}}(S, \gamma^*) = \{ T \in \mathcal{M}_{2,\text{rev}}(S), \gamma^*(T) \geq \gamma^*\}.
\end{equation}
The goal is to characterize the sample complexity of entropy rate estimation as a function of $S$, $\gamma^*$, and the estimation accuracy.

Note that the entropy rate of a first-order Markov chain can be written as
\begin{equation}
\bar{H} = \sum_{i = 1}^S \pi_i H(X_2|X_1 = i). 
\end{equation}
Given a sample path $\mathbf{X} = (X_0,X_1,\ldots,X_n)$, let  $\hat{\pi}$ denote the empirical distribution of states, and the subsequence of $\mathbf{X}$ containing elements \emph{following} any occurrence of the state $i$ as
$
\mathbf{X}^{(i)} = \{ X_j:  X_j \in \mathbf{X}, X_{j-1} = i, j\in [n]\}.
$
A natural idea to estimate the entropy rate $\bar{H}$ is to use $\hat{\pi}_i$ to estimate $\pi_i$ and an appropriate Shannon entropy estimator to estimate $H(X_2|X_1 = i)$. We define two estimators: 
\begin{enumerate}
	\item The \emph{empirical entropy rate}: $\bar{H}_{\emp} = \sum_{i = 1}^S \hat{\pi}_i \hat H_{\mathsf{emp}}\!\left( \mathbf{X}^{(i)} \right)$. Note that $\hat{H}_{\emp}(\mathbf{Y})$ computes the Shannon entropy of the empirical distribution of its argument $\mathbf{Y} = (Y_1,Y_2,\ldots,Y_m)$.
	\item Our entropy rate estimator: $\Hopt = \sum_{i = 1}^S \hat{\pi}_i \hat{H}_{\opt}\!\left (\mathbf{X}^{(i)} \right )$, where $\hat{H}_{\opt}$ is any minimax rate-optimal Shannon entropy estimator designed for i.i.d.\ data, such as those in~\cite{Valiant--Valiant2011power,Wu--Yang2014minimax,Jiao--Venkat--Han--Weissman2015minimax}.
\end{enumerate} 


\section{Main results}
\label{sec.mainresults}

Our first result provides performance guarantees for the empirical entropy rate $\Hemp$ and our entropy rate estimator $\Hopt$: 

\begin{theorem}\label{thm.mainresultsupper}
Suppose $(X_0,X_1,\ldots,X_n)$ is a sample path from a stationary reversible Markov chain with spectral gap $\gamma$. If $S^{0.01}\lesssim n\lesssim S^{2.99}$ and $\frac{1}{\gamma} \lesssim \frac{S}{\ln n\ln^2 S } \wedge \frac{S^3}{n \ln n \ln^3 S}$, there exists some constant $C>0$ independent of $n,S,\gamma$ such that the entropy rate estimator $\Hopt$ satisfies:\footnote{The asymptotic results in this section are interpreted by parameterizing $n=n_S$ and $\gamma=\gamma_S$ and $S \to \infty$ subject to the conditions of each theorem.} as $S\to\infty$,
\begin{align}
P\left( |\Hopt - \bar{H}| \le C\frac{S^2}{n\ln S} \right) \to 1
\end{align}

Under the same conditions, there exists some constant $C'>0$ independent of $n,S,\gamma$ such that the empirical entropy rate $\Hemp$ satisfies: as $S\to\infty$,
\begin{align}
P\left( |\Hemp - \bar{H}| \le C'\frac{S^2}{n} \right) \to 1.
\end{align}
\end{theorem}



Theorem~\ref{thm.mainresultsupper} shows that when the sample size is not too large, and the mixing is not too slow, it suffices to take $n \gg \frac{S^2}{\ln S}$ for the estimator $\bar{H}_{\mathsf{opt}}$ to achieve a vanishing error, and $n\gg S^2$ for the empirical entropy rate. Theorem~\ref{thm.mainresultsupper} improves over~\cite{kamath2016estimation} in the analysis of the empirical entropy rate in the sense that unlike the error term $O(\frac{S^2}{n\gamma})$, our dominating term $O(\frac{S^2}{n})$ does not depend on the mixing time. 

Note that we have made mixing time assumptions in the upper bound analysis of the empirical entropy rate in Theorem~\ref{thm.mainresultsupper}, which is natural since~\cite{kamath2016estimation} showed that it is necessary to impose mixing time assumptions to provide meaningful statistical guarantees for entropy rate estimation in Markov chains. The following result shows that mixing assumptions are only needed to control the variance of the empirical entropy rate: the bias of the empirical entropy rate vanishes uniformly over all Markov chains regardless of reversibility and mixing time assumptions as long as $n\gg S^2$. 

\begin{theorem}\label{thm.universalbias}
	Let $n,S\geq 1$. Then, 
	\begin{align}
	\sup_{T\in \mathcal{M}_2(S)} |\bar{H} - \mathbb{E}[\Hemp]| & \leq  \frac{2S^2}{n}\ln\left( \frac{n}{S^2} +1 \right) + \frac{(S^2+2)\ln 2}{n}.
	\end{align}
\end{theorem}

Theorem~\ref{thm.universalbias} implies that if $n\gg S^2$, the bias of the empirical entropy rate estimator universally vanishes for any stationary Markov chains. 

Now we turn to the lower bounds, which show that the scalings in Theorem \ref{thm.mainresultsupper} are in fact tight. The next result shows that the bias of the empirical entropy rate $\bar{H}_{\emp}$ is non-vanishing unless $n\gg S^2$, even when the data are independent.


%

\begin{theorem}\label{thm.biasmlebound}
If $\{X_0,X_1,\ldots,X_n\}$ are mutually independent and uniformly distributed, then
\begin{align}
|\bar{H} - \mathbb{E}[\Hemp]| & \geq \ln \left(\frac{S^2}{n+ S-1} \right).
\end{align}
\end{theorem}
The following corollary is immediate.


\begin{corollary}\label{cor.empirical}
There exists a universal constant $c>0$ such that when $n\leq c S^2$, the absolute value of the bias of $\bar{H}_{\emp}$ is bounded away from zero even if the Markov chain is memoryless.
\end{corollary}


The next theorem presents a minimax lower bound for entropy rate estimation which applies to any estimation scheme regardless of its computational cost. In particular, it shows that $\Hopt$ is minimax rate-optimal under mild assumptions on the mixing time.

\begin{theorem}\label{thm.mainresultslower}
For $n\ge \frac{S^2}{\ln S}, \ln n\ll\frac{S}{(\ln S)^2}, \gamma^*\le 1-C_2\sqrt{\frac{S\ln^3 S}{n}}$, we have
\begin{align}
\liminf_{S \to\infty}
\inf_{\hat{H}} \sup_{T \in \mathcal{M}_{2,\text{rev}}(S,\gamma^*)}P\left( |\hat{H} - \bar{H}| \geq C_1\frac{S^2}{n\ln S} \right) \geq \frac{1}{2}.
\end{align}
Here $C_1,C_2$ are universal constants from Theorem~\ref{thm.lowerboundpoisson}.
\end{theorem}

%

The following corollary, which follows from Theorem~\ref{thm.mainresultsupper} and~\ref{thm.mainresultslower}, presents the critical scaling that determines whether consistent estimation of the entropy rate is possible.
\begin{corollary}\label{cor.samplecomplexitycor}
If $\frac{\ln^3 S}{S}\ll \gamma^*\le 1-C_2\frac{\ln^2 S}{\sqrt{S}}$, there exists an estimator $\hat{H}$ which estimates the entropy rate with a uniformly vanishing error over Markov chains $\mathcal{M}_{2,\text{rev}}(S,\gamma^*)$ if and only if $n\gg \frac{S^2}{\ln S}$.
\end{corollary}

To conclude this section we summarize our result in terms of the sample complexity for estimating the entropy rate within a few bits ($\epsilon=\Theta(1)$), classified according to the relaxation time:

\begin{itemize}
	\item $\trel=1$: this is the i.i.d.~case and the sample complexity is $\Theta(\frac{S}{\ln S})$;

	\item $1<\trel\ll 1 + \Omega(\frac{\ln^2 S}{\sqrt{S}})$: in this narrow regime the sample complexity is at most $O(\frac{S^2}{\ln S})$ and no matching lower bound is known;

	\item $1 + \Omega(\frac{\ln^2 S}{\sqrt{S}}) \leq \trel \ll \frac{S}{\ln^3 S} $: the sample complexity is $\Theta(\frac{S^2}{\ln S})$;

	\item $\trel \gtrsim \frac{S}{\ln^3 S}$: the sample complexity is $\Omega(\frac{S^2}{\ln S})$ and no matching upper bound is known. In this case the chain mixes very slowly and it is likely that the variance will dominate.
\end{itemize}

\section{Sketch of the proof}\label{sec.proof_sketch}
In this section we sketch the proof of Theorems \ref{thm.mainresultsupper}, \ref{thm.universalbias} and \ref{thm.mainresultslower}, and defer the details to the appendix. 

\subsection{Proof of Theorem~\ref{thm.mainresultsupper}}
A key step in the analysis of $\Hemp$ and $\Hopt$ is the idea of simulating a finite-state Markov chain from independent samples~\cite[p.~19]{billingsley1961statistical}: consider an independent collection of random variables $X_0$ and $W_{in}$ ($i = 1,2,\ldots,S;n = 1,2,\ldots $) such that $
P_{X_0}(i) = \pi_i, 
P_{W_{in}}(j) = T_{ij}.$ Imagine the variables $W_{in}$ set out in the following array:
\begin{equation*}
\begin{matrix}
W_{11} &W_{12}& \ldots & W_{1n}& \ldots \\
W_{21} &W_{22}& \ldots & W_{2n}& \ldots \\
\vdots &\vdots& \ddots & \vdots & \vdots \\
W_{S1} &W_{S2}&\ldots &W_{Sn} &\ldots
\end{matrix}
\end{equation*}
First, $X_0$ is sampled. If $X_0 = i$, then the first variable in the $i$th row of the array is sampled, and the result is assigned by definition to $X_1$. If $X_1 = j$, then the first variable in the $j$th row is sampled, unless $j = i$, in which case the second variable is sampled. In any case, the result of the sampling is by definition $X_2$. The next variable sampled is the first one in row $X_2$ which has not yet been sampled. This process thus continues. After collecting $\{X_0,X_1,\ldots,X_n\}$ from the model, we assume that the last variable sampled from row $i$ is $W_{in_i}$. It can be shown that observing a Markov chain $\{X_0,X_1,\ldots,X_n\}$ is equivalent to observing $\{X_0, \{W_{ij}\}_{i\in [S], j\in [n_i]}\}$, and consequently $\hat{\pi}_i = n_i/n, \mathbf{X}^{(i)} = (W_{i1},\ldots,W_{in_i})$.

The main reason to introduce the above framework is to analyze $\hat{H}_{\mathsf{emp}}(\mathbf{X}^{(i)})$ and $\hat{H}_{\mathsf{opt}}(\mathbf{X}^{(i)})$ as if the argument $\mathbf{X}^{(i)}$ is an i.i.d. vector. Specifically, although $W_{i1},\cdots, W_{im}$ conditioned on $n_i=m$ are not i.i.d., they are i.i.d. as $T_i$ for any \emph{fixed} $m$. Hence, using the fact that each $n_i$ concentrates around $n\pi_i$ (cf. Definition \ref{def.goodevents} and Lemma \ref{lemma.goodeventshighprobability} for details), we may use the concentration properties of $\hat{H}_{\mathsf{emp}}$ and $\hat{H}_{\mathsf{opt}}$ (cf. Lemma \ref{lemma.concentrationentropy}) on i.i.d. data for each \emph{fixed} $m\approx n\pi_i$ and apply the union bound in the end. 

Based on this alternative view, we have the following theorem, which implies Theorem \ref{thm.mainresultsupper}. 
\begin{theorem}\label{thm.upperbound}
	Suppose $(X_0,X_1,\ldots,X_n)$ comes from a stationary reversible Markov chain with spectral gap $\gamma$. Then, with probability tending to one, the entropy rate estimators satisfy
	\begin{align}
	| \Hopt - \bar{H} | &\lesssim \frac{S^2}{n\ln S} + \left( \frac{S}{n} \right)^{0.495} + \frac{S \ln S}{n^{0.999}} + \frac{S \ln S \ln n}{n\gamma} +  \sqrt{\frac{S \ln n \ln^2 S}{n\gamma}},
	\label{eq:optmain} \\
	| \Hemp - \bar{H} | &\lesssim \frac{S^2}{n} + \left( \frac{S}{n} \right)^{0.495} + \frac{S \ln S}{n^{0.999}} + \frac{S \ln S \ln n}{n\gamma} +  \sqrt{\frac{S \ln n \ln^2 S}{n\gamma}}.
	\label{eq:empmain}
	\end{align}
\end{theorem}

\subsection{Proof of Theorem~\ref{thm.universalbias}}
By the concavity of entropy, the empirical entropy rate $\Hemp$ underestimates the truth $\bar{H}$ in expectation. On the other hand, the average codelength $\bar{L}$ of any lossless source code is at least $\bar{H}$ by Shannon's source coding theorem. As a result, $\bar{H} - \bE[\Hemp]\le \bar{L} - \bE[\Hemp]$, and we may find a good lossless code to make the RHS small. 

Since the conditional empirical distributions maximizes the likelihood for Markov chains (Lemma~\ref{lemma.representationofempiricalentropy}), we have
\begin{align}
\mathbb{E}_P \left[ \frac{1}{n} \ln \frac{1}{Q_{X_1^n|X_0}(X_1^n|X_0)} \right] & \geq \mathbb{E}_P \left[  \frac{1}{n} \ln \frac{1}{P_{X_1^n|X_0}(X_1^n|X_0)} \right] = \bar{H} \\
& \geq \mathbb{E}_P \left[ \min_{P\in \mathcal{M}_2(S)} \frac{1}{n} \ln \frac{1}{P_{X_1^n|X_0}(X_1^n|X_0)} \right] = \mathbb{E}[\Hemp]
\end{align}
where $\cM_2(S)$ denotes the space of all first-order Markov chains with state $[S]$. Hence,
\begin{align}\label{eqn.mlebiasboundbycompression}
|\bar{H} -\mathbb{E}[\Hemp]| & \leq \inf_Q \sup_{P\in \mathcal{M}_2(S), x_0^n} \frac{1}{n}\ln \frac{P(x_1^n|x_0)}{Q(x_1^n|x_0)} .
\end{align}

The following lemma provides a non-asymptotic upper bound on the RHS of~(\ref{eqn.mlebiasboundbycompression}) and completes the proof of Theorem \ref{thm.universalbias}.
\begin{lemma}\cite{Tatwawadi--Jiao--Weissman17}\label{lemma.worstcaseredundancymarkov}
Let $\mathcal{M}_{2}(S)$ denote the space of Markov chains with alphabet size $S$ for each symbol. Then, the worst case minimax redundancy is bounded as
\begin{align}
\inf_Q \sup_{P \in \mathcal{M}_2(S),x_0^n} \frac{1}{n} \ln \frac{P(x_1^n|x_0)}{Q(x_1^n|x_0)} & \leq \frac{2S^2}{n}\ln\left( \frac{n}{S^2} +1 \right) + \frac{(S^2+2)\ln 2}{n}.
\end{align}
\end{lemma}

\subsection{Proof of Theorem \ref{thm.mainresultslower}}
To prove the lower bound for Markov chains, we first introduce an auxiliary model, namely, the \emph{independent Poisson} model and show that the sample complexity of the Markov chain model is lower bounded by that of the independent Poisson model. Then we apply the so-called method of fuzzy hypotheses \cite[Theorem 2.15]{Tsybakov2008} (see also \cite[Lemma 11]{HJWW17}) to prove a lower bound for the independent Poisson model. 
We introduce the independent Poisson model below, which is parametrized by an $S\times S$ symmetric matrix $R$, an integer $n$ and a parameter $\lambda> 0$.

\begin{definition}[Independent Poisson model]
	Given an $S\times S$ symmetric matrix $R=(R_{ij})$ with $R_{ij} \geq 0$ and a parameter $\lambda >0$, under the independent Poisson model, we observe $X_0\sim \pi=\pi(R)$, and an $S\times S$ matrix $C=(C_{ij})$ with independent entries distributed as $C_{ij} \sim \mathsf{Poi}\left( \lambda R_{ij} \right)$, where
	\begin{align}
	\pi_i=\pi_i(R) = \frac{r_i}{r}, \quad r_i = \sum_{j =1}^S R_{ij},\quad r=\sum_{i=1}^S r_i.
	\label{eq:piR}
	\end{align}
\end{definition}
For each symmetric matrix $R$, by normalizing the rows we can define a transition matrix $T=T(R)$ of a \emph{reversible} Markov chain with stationary distribution $\pi=\pi(R)$. Upon observing the Poisson matrix $C$, the functional to be estimated is the entropy rate $\bar{H}$ of $T(R)$. Given $\tau>0$ and $\gamma,q\in (0,1)$,  define the following collection of  symmetric matrices:
\begin{align} \label{eqn.ipuncertaintysetdef}
\calR(S,\gamma, \tau,q)
& = \Bigg \{ R \in \reals_+^{S\times S}: R = R^{\top}, \gamma^*(T) \geq \gamma, \sum_{i,j} R_{ij} \geq \tau, \pi_{\min} \geq q \Bigg  \},
\end{align}
where $\pi_{\min} = \min_i \pi_i$. The reduction to independent Poisson model is summarized below: 
\begin{lemma}\label{lemma.poissontomc}
	If there exists an estimator $\hat{H}_1$ for the Markov chain model with parameter $n$ such that $\bP(|\hat{H}_1-\bar{H}|\ge \epsilon)\le \delta$ under any $T\in\cM_{2,{\rm rev}}(S,\gamma)$, 
	then there exists another estimator $\hat{H}_2$ for the independent Poisson model with parameter $\lambda=\frac{4 n}{\tau}$ such that
	\begin{align}
	\sup_{R \in \calR(S,\gamma,\tau,q) } \bP\left( |\hat{H}_2 - \bar{H}(T(R))| \geq  \epsilon \right) & \leq  \delta +  2Sn^{-\frac{c_3^2}{4+10c_3}} + S n^{-c_3/2},
	\end{align}
	provided $q \geq \frac{ c_3 \ln n}{n \gamma}$, where $c_3\geq 20$ is a universal constant. 
\end{lemma}

To prove the lower bound for the independent Poisson model, the goal is to construct two symmetric random matrices (whose distributions serve as the priors), such that 
(a) they are sufficiently concentrated near the desired parameter space $\calR(S,\gamma,\tau,q)$ for properly chosen parameters $\gamma,\tau,q$;
(b) their entropy rates are separated;
	(c) the induced marginal laws of the sufficient statistic $\bfC=X_0\cup \{ C_{ij} + C_{ji}: i\neq j, 1\leq i\leq j\leq S\} \cup \{C_{ii}: 1\leq i\leq S\}$ are statistically indistinguishable.
The prior construction in Definition \ref{con.priorconstruction} satisfies all these three properties (cf. Lemmas \ref{lemma.indistinguishableipmodel}, \ref{lemma.functionalseperation}, \ref{lemma.spectralgapcontrol}), and thereby lead to the following theorem:
\begin{theorem}\label{thm.lowerboundpoisson}
	If $n \geq \frac{S^2}{\ln S}, \ln n\ll \frac{S}{(\ln S)^2}, \gamma^* \leq  1 - C_2\sqrt{\frac{S\ln^3 S}{n}}$, we have
	\begin{align}
	\liminf_{S\to\infty} \inf_{\hat{H}} \sup_{R \in \calR(S,\gamma^*,\tau,q)} \bP \left( | \hat{H} - \bar{H}| \ge C_1 \frac{S^2}{n\ln S} \right) \geq \frac{1}{2}
	\end{align}
	where $\tau=S, q=\frac{1}{5\sqrt{n\ln S}}$, and $C_1, C_2>0$ are two universal constants. 
\end{theorem}

\section{Application: Fundamental limits of language modeling}\label{sec.languagemodels}

In this section, we apply entropy rate estimators to estimate the fundamental limits of language modeling. 
A language model specifies the joint probability distribution of a sequence of words, $Q_{X^n}(x^n)$. It is common to use a $(k-1)$th-order Markov assumption to train these models, using sequences of $k$ words (also known as $k$-grams,\footnote{In the language modeling literature these are typically known as $n$-grams, but we use $k$ to avoid conflict with the sample size.} sometimes with Latin prefixes \textit{unigrams}, \textit{bigrams}, \textit{etc.}), with values of $k$ of up to 5 \cite{Jurafsky:2009:SLP:1214993}. 
A commonly used metric to measure the efficacy of a model $Q_{X^n}$ is the \emph{perplexity} (whose logarithm is called the \emph{cross-entropy rate}): 
\[
    \mathrm{perplexity}_Q\left(X^n\right) = \sqrt[n]{\frac{1}{Q_{X^n}(X^n)}}.
\]
If a language is modeled as a stationary and ergodic stochastic process with entropy rate $\bar H$, and $X^n$ is drawn from the language with true distribution $P_{X^n}$, then \cite{kieffer1991sample}
\[
    \bar{H} \le \liminf_{n \rightarrow \infty} \frac1n \log \frac{1}{Q_{X^n}(X^n)} = \liminf_{n \rightarrow \infty} \log \left[\mathrm{perplexity}_Q\left(X^n\right)\right],
\]
with equality when $Q = P$. In this section, all logarithms are with respect to base $2$ and all entropy are measured in bits.

The entropy rate of the English language is of significant interest to language model researchers: since $2^{\bar H}$ is a tight lower bound on perplexity, this quantity indicates how close a given language model is to the optimum. Several researchers have presented estimates in bits per character \cite{Shannon1951prediction,cover1978convergent,Brown:1992:EUB:146680.146685}; because language models are trained on words, these estimates are not directly relevant to the present task. In one of the earliest papers on this topic, Claude Shannon \cite{Shannon1951prediction} gave an estimate of 11.82 bits per word. This latter figure has been comprehensively beaten by recent models; for example, \cite{DBLP:journals/corr/KuchaievG17} achieved a perplexity corresponding to a cross-entropy rate of 4.55 bits per word.

To produce an estimate of the entropy rate of English, we used two well-known linguistic corpora: the Penn Treebank (PTB) and Google's One Billion Words (1BW) benchmark. Results based on these corpora are particularly relevant because of their widespread use in training models. We used the conditional approach proposed in this paper with the JVHW estimator describe in Section \ref{sec.experiments}. The PTB corpus contains about $n \approx 1.2$ million words, of which $S \approx 47,000$ are unique. The 1BW corpus contains about $n \approx 740$ million words, of which $S \approx 2.4$ million are unique.



We estimate the conditional entropy $H(X_k|X^{k-1})$ for $k=1,2,3,4$, and our results are shown in Figure~\ref{fig:language-entropy}. The estimated conditional entropy $\hat H(X_k|X^{k-1})$ provides us with a refined analysis of the intrinsic uncertainty in language prediction with context length of only $k-1$. For 4-grams, using the JVHW estimator on the 1BW corpus, our estimate is 3.46 bits per word. With current state-of-the-art models trained on the 1BW corpus having an cross-entropy rate of about 4.55 bits per word \cite{DBLP:journals/corr/KuchaievG17}, this indicates that language models are still at least 0.89 bits per word away from the fundamental limit. (Note that since $H(X_k|X^{k-1})$ is decreasing in $k$, $H(X_4|X^3) > \bar H$.) Similarly, for the much smaller PTB corpus, we estimate an entropy rate of 1.50 bits per word, compared to state-of-the-art models that achieve a cross-entropy rate of about 5.96 bits per word \cite{DBLP:journals/corr/ZophL16}, at least 4.4 bits away from the fundamental limit.

More detailed analysis, e.g., the accuracy of the JVHW estimates, is shown in the Appendix \ref{sec.languagemodels_appendix}.



\pgfplotsset{entropyrateplot/.style={%
        width=\columnwidth,
        xticklabel style={font=\footnotesize},
        yticklabel style={font=\footnotesize},
        xtick={1,2,3,4},
        xlabel style={font=\footnotesize},
        ylabel style={font=\footnotesize},
        ymajorgrids=true,
        legend style={at={(0.98,0.98),font=\footnotesize},anchor=north east},
        xlabel={memory length $k$},
        ylabel={estimated cond.\ entropy $\hat H(X_k|X^{k-1}_1)$},
        xmin=0, xmax=7, ymin=0,
    }
}
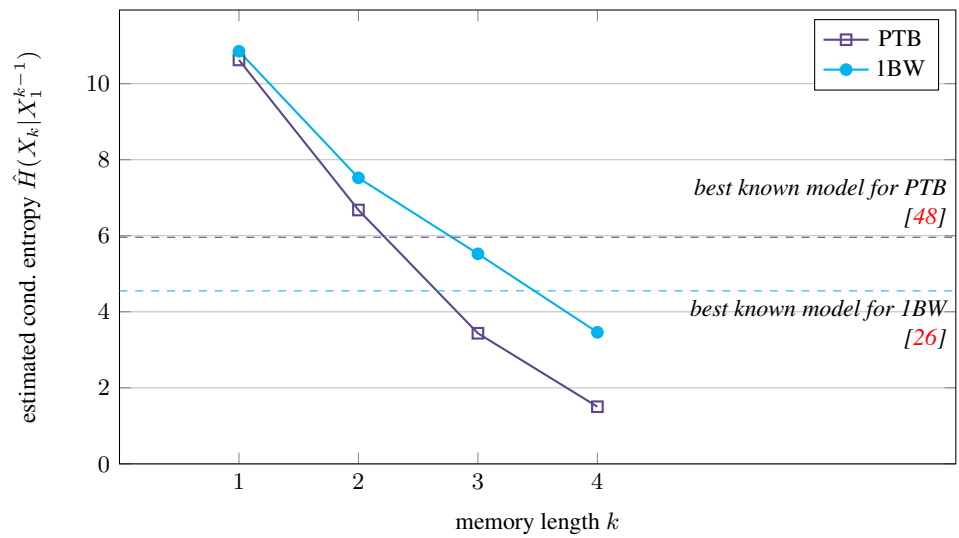
\begin{figure}[h]
\centering
\begin{tikzpicture}
\begin{axis}[entropyrateplot, height=3in, width=5in]
    \addplot[thick, color=MediumPurple4, mark=square]
        table[x=N,y=estHrate] { 
            N    estHrate
            1    10.6244158412
            2    6.67734824421
            3    3.43519634366
            4    1.50354321094
    };
    \addplot[thick, color=DeepSkyBlue2, mark=*]
        table[x=N,y=estHrate] { 
            N    estHrate
            1    10.8521979168
            2    7.52302384907
            3    5.52436031624
            4    3.46099824872
    };
    \addplot[domain=0:7,dashed,color=MediumPurple4]{5.96};
    \node[anchor=south east, align=right] at (axis cs:7,5.96) {\itshape\footnotesize best known model for PTB \\ \itshape\footnotesize \cite{DBLP:journals/corr/ZophL16}};
    \addplot[domain=0:7,dashed,color=DeepSkyBlue2]{4.55};
    \node[anchor=north east, align=right] at (axis cs:7,4.55) {\itshape\footnotesize best known model for 1BW \\ \itshape\footnotesize \cite{DBLP:journals/corr/KuchaievG17}};
    \legend{{PTB JVHW},{1BW JVHW}}
    \legend{{PTB},{1BW}}
\end{axis}
\end{tikzpicture}
\caption{Estimates of conditional entropy based on linguistic corpora}
\label{fig:language-entropy}
\end{figure}

%% file: appendix.tex
\section{Proof of Theorem \ref{thm.mainresultsupper}}

\subsection{Concentration of $\hat{H}_{\opt}$ and $\hat{H}_{\emp}$}

The performance of $\hat{H}_{\opt}$ and $\hat{H}_{\emp}$ in terms of Shannon entropy estimation is collected in the following lemma.

\begin{lemma}\label{lemma.concentrationentropy}
	Suppose $\alpha = 0.001, \alpha' = 0.01$ and one observes $n$ i.i.d.\ samples $X_1,X_2,\ldots,X_n \stackrel{\mathrm{i.i.d.}}{\sim} P$. Then, there exists an entropy estimator $\hat{H}_{\opt} =\hat{H}_{\opt}(X_1,\ldots,X_n) \in [0, \ln S]$ such that for any $t > 0$,
	\begin{align}\label{eqn.concentrationoptimalestimatoriid}
	P\left( |\hat{H}_{\opt} - H(P)|\geq t + c_2 \frac{S}{n \ln S} \right) & \leq 2 \exp\left( - c_1 t^2 n^{1-\alpha} \right),
	\end{align}
	where $c_1,c_2>0$ are universal constants, and $H(P)$ is the Shannon entropy.
	Moreover, the empirical entropy $\hat{H}_{\emp} = \hat{H}_{\emp}(X_1,X_2,\ldots,X_n) \in [0, \ln S]$ satisfies, for any $t>0$,
	\begin{align}
	P\left( |\hat{H}_{\emp} - H(P)|\geq t + c_2 \frac{S}{n } \right) & \leq 2 \exp\left( - c_1 t^2 n^{1-\alpha} \right).
	\end{align}
	Consequently, for any $\beta >0$,
	\begin{align}\label{eqn.alphapidef}
	P\left( |\hat{H}_{\opt} - H(P)| \geq \frac{c_2 S}{n \ln S} + \sqrt{\frac{\beta}{c_1 n^{1-\alpha'}}} \right) & \leq \frac{2}{n^\beta},
	\end{align}
	and
	\begin{align}
	P\left( |\bar{H}_{\emp}) - H(P)| \geq \frac{c_2 S}{n } + \sqrt{\frac{\beta}{c_1 n^{1-\alpha'}}} \right) & \leq \frac{2}{n^\beta}.
	\end{align}
\end{lemma}

\begin{proof}
	The part pertaining to the concentration of $\hat{H}_{\opt}$ follows from~\cite{Wu--Yang2014minimax, Jiao--Venkat--Han--Weissman2015minimax,acharya2016unified}. The part pertaining to the empirical entropy follows from~\cite{Antos--Kontoyiannis2001convergence},\cite[Proposition 1]{Paninski2003},\cite[Eqn. (88)]{Jiao--Venkat--Weissman2014MLE}.
\end{proof}

\subsection{Analysis of \texorpdfstring{$\Hopt$}{Hopt} and \texorpdfstring{$\Hemp$}{Hemp}}

Next we define two events that ensure the proposed entropy rate estimator $\Hopt$ and the empirical entropy rate $\Hemp$ is accurate, respectively:

\begin{definition}[``Good'' event in estimation]\label{def.goodevents}
	Let $0<c_4<1$ and $c_3 \geq 20$ be some universal constants. We take $c_4 = 0.001$.
	\begin{enumerate}
		\item For every $i, 1\leq i\leq S$, define the event
		\begin{equation}
		\mathcal{E}_i =  \left \{ |\hat{\pi}_i - \pi_i| \leq  c_3 \max \left \{ \frac{\ln n}{n\gamma}, \sqrt{\frac{\pi_i \ln n}{n \gamma}} \right \} \right \}
		\label{eq:Ei}
		\end{equation}
		\item For every $i \in [S]$ such that $\pi_i \geq n^{c_4-1} \vee 100c_3^2 \frac{\ln n}{n\gamma}$, define the event $\calH_i$ as
		\begin{equation}\label{eqn.goodeventiidentropy}
		|\hat{H}_{\opt}(W_{i1},W_{i2},\ldots, W_{im
		}) - H_i | \leq \frac{c_2 S}{m \ln S} + \sqrt{\frac{\beta }{c_1 m^{1-\alpha'}}},
		\end{equation}
		for \emph{all} $m$ such that $n\pi_i - c_3 \sqrt{\frac{n\pi_i \ln n}{\gamma}} \leq m \leq   n\pi_i + c_3 \sqrt{\frac{n\pi_i \ln n}{\gamma}} $,
		where $\beta = \frac{c_3^2}{4 + 10c_3}$, $c_1,c_2,\alpha'$ are from Lemma~\ref{lemma.concentrationentropy}.
	\end{enumerate}
	Finally, define the ``good'' event as the intersection of all the events above:
	\begin{equation}
	\calG_\opt \triangleq \pth{\bigcap_{i\in [S]} \calE_i} \cap \Bigg( \bigcap_{i: \pi_i \geq n^{c_4-1} \vee 100c_3^2 \frac{\ln n}{n\gamma}} \calH_i \Bigg).
	\end{equation}
	Analogously, we define the ``good'' event $\calG_\emp$ for the empirical entropy rate $\Hemp$ in a similar fashion with (\ref{eqn.goodeventiidentropy}) replaced by
	\begin{equation}
	|\hat{H}_{\emp}(W_{i1},W_{i2},\ldots, W_{im
	}) - H_i | \leq \frac{c_2 S}{m} + \sqrt{\frac{\beta }{c_1 m^{1-\alpha'}}}.
	\end{equation}
\end{definition}

The following lemma shows that the ``good'' events defined in Definition~\ref{def.goodevents} indeed occur with high probability.

\begin{lemma}\label{lemma.goodeventshighprobability}
	Both $\calG_{\opt}$ and $\calG_{\emp}$ in Definition~\ref{def.goodevents} occur with probability at least
	\begin{equation}
	1 - \frac{2S}{n^\beta} - \frac{4c_3 (10)^\beta}{9^\beta} \frac{S}{n^{c_4 (\beta -1)}},
	\label{eqn.probabilitylowerbound}
	\end{equation}
	where $\beta = \frac{c_3^2}{4 + 10c_3},c_3\geq 20$.
\end{lemma}

\begin{proof}[Proof of Theorem \ref{thm.upperbound}]
Pick $c_4=0.001, \alpha'=0.01$. We write
	\begin{align}
	\bar{H} & = \sum_{i = 1}^S \pi_i H_i, \\
	\bar{H}_{\mathsf{opt}} & = \sum_{i =1}^S \hat{\pi}_i \hat{H}_i,
	\end{align}
	where $H_i = H(X_2|X_1 = i), \hat{H}_i = \hat{H}_{\opt}(\mathbf{X}^{(i)}) = \hat{H}_{\opt}(W_{i1},\ldots,W_{in_i})$.
	Write
	\begin{align}
	\bar{H}_{\mathsf{opt}} - \bar{H} & = \underbrace{\sum_{i = 1}^S \pi_i \left( \hat{H}_i - H_i \right)}_{E_1} + \underbrace{\sum_{i = 1}^S \hat{H}_i(\hat{\pi}_i - \pi_i)}_{E_2}.
	\end{align}
	Next we bound the two terms separately under the condition that the ``good'' event $\calG_\opt$ in Definition~\ref{def.goodevents} occurs.
	
	Note that the function $\pi_i \mapsto n\pi_i - c_3 \sqrt{\frac{n\pi_i \ln n}{\gamma}}$ is an increasing function when $\pi_i \geq \frac{100c_3^2 \ln n}{n\gamma}$. Thus we have
	\begin{align}
	\label{eq:niminus}
	n\pi_i - c_3 \sqrt{\frac{n\pi_i \ln n}{\gamma}} = n\pi_i \left( 1 - c_3 \sqrt{\frac{\ln n}{n\pi_i \gamma}} \right)  \geq \frac{9}{10} n\pi_i ,
	\end{align}
	whenever $\pi_i \geq \frac{100c_3^2 \ln n}{n\gamma}$.
	
	Let $\epsilon(m) \triangleq \frac{c_2 S}{m \ln S} + \sqrt{\frac{\beta }{c_1 m^{1-\alpha'}}}$, which is decreasing in $m$.
	Let $n_i^{\pm} \triangleq n\pi_i \pm c_3 \max \left \{ \frac{\ln n}{\gamma}, \sqrt{\frac{n \pi_i \ln n}{\gamma}} \right \} $.
	Note that for each $i\in [S]$,
	\begin{align*}
	\sth{|\hat H_i - H_i| \leq \epsilon(n_i)}
	\supset & ~  \sth{|\hat H_i - H_i| \leq \epsilon(n_i), |\hat \pi_i - \pi_i| \leq \max \left \{ \frac{\ln n}{n \gamma}, \sqrt{\frac{\pi_i \ln n}{n \gamma}} \right \}}\\
	= & ~  \sth{|\hat H_\opt(W_{i1},\ldots,W_{in_i}) - H_i| \leq \epsilon(n_i),  n_i^- \leq n_i \leq n_i^+}\\
	\supset & ~  \bigcap_{m=n_i^-}^{n_i^+} \{|\hat H_\opt(W_{i1},\ldots,W_{im}) - H_i| \leq \epsilon(m)\}.
	\end{align*}
	The key observation is that for each fixed $m$, $W_{i1},\ldots,W_{im}$ are i.i.d.\ as $T_{i}$.\footnote{Note that effectively we are taking a union over the value of $n_i$ instead of conditioning. In fact, conditioned on $n_i=m$, $W_{i1},\ldots,W_{im}$ are no longer i.i.d.\ as $T_{i}$.}
	Taking the intersection over $i\in [S]$, we have
	\[
	\sth{|\hat H_i - H_i| \leq \epsilon(n_i), ~i=1,\ldots,S}  \supset \calG_\opt.
	\]
	Therefore, on the event $\calG_\opt$, we have
	\begin{align*}
	|E_1| & \leq \sum_{i = 1}^S \pi_i |\hat{H}_i - H_i| \\
	& \leq \sum_{i: \pi_i \geq n^{c_4-1} \vee 100c_3^2 \frac{\ln n}{n\gamma}} \pi_i |\hat{H}_i - H_i| + \sum_{i: \pi_i \leq n^{c_4-1} \vee 100c_3^2 \frac{\ln n}{n\gamma}}\pi_i |\hat{H}_i - H_i| \\
	& \overset{\prettyref{eq:niminus}}{\leq} \sum_{i: \pi_i \geq n^{c_4-1} \vee 100c_3^2 \frac{\ln n}{n\gamma}} \pi_i \left( \frac{c_2 S}{ 0.9 n\pi_i \ln S} + \sqrt{\frac{\beta}{c_1 (0.9 n\pi_i )^{1-\alpha'}  }} \right) \nonumber \\
	& \qquad + \sum_{i: \pi_i \leq n^{c_4-1} \vee 100c_3^2 \frac{\ln n}{n\gamma}} \pi_i \ln S \\
	& \lesssim \frac{S^2}{n\ln S} + \left( \frac{S}{n} \right)^{\frac{1-\alpha'}{2}} + \frac{S \ln S}{n^{1-c_4}} \vee \frac{S \ln S \ln n}{n\gamma}, \numberthis\label{eqn.upperbound.e1}
	\end{align*}
	where the last step follows from \prettyref{eq:niminus} and the fact that $\sum_{i\in[S]} \pi_i^\alpha \leq S^{1-\alpha}$ for any $\alpha \in [0,1]$.
	As for $E_2$, on the event $\calG_\opt$, we have
	\begin{equation}
	\label{eqn.upperbound.e2}
	|E_2| \leq \sum_{i = 1}^S \hat{H}_i |\hat{\pi}_i - \pi_i|
	\leq \ln S \sum_{i =1}^S c_3 \max \left \{ \frac{\ln n}{n\gamma}, \sqrt{\frac{\pi_i \ln n}{n \gamma}} \right \}
	\lesssim \frac{S \ln S \ln n}{n\gamma} \vee  \sqrt{\frac{S \ln n \ln^2 S}{n\gamma}}.
	\end{equation}
	Combining \eqref{eqn.upperbound.e1} and \eqref{eqn.upperbound.e2}, and using Lemma \ref{lemma.goodeventshighprobability}, completes the proof of \prettyref{eq:optmain}. The proof of \prettyref{eq:empmain} follows entirely analogously with $\calG_\opt$ replaced by $\calG_\emp$.
\end{proof}

\section{Proof of Theorem \ref{thm.biasmlebound}}
\label{sec.empiricallowerbound}

We first prove Theorem~\ref{thm.biasmlebound}, which quantifies the performance limit of the empirical entropy rate. Lemma~\ref{lemma.representationofempiricalentropy} in Section~\ref{sec.auxiliarylemmas} shows that
\begin{equation}
\bar{H}_{\emp}  = \min_{P\in \mathcal{M}_2(S)} \frac{1}{n} \ln \frac{1}{P_{X_1^n|X_0}(X_1^n|X_0)},
\end{equation}
where $\mathcal{M}_2(S)$ denotes the set of all Markov chain transition matrices with state space $\mathcal{X}$ of size $S$. Since
\begin{equation}
\bar{H} = \mathbb{E}_P \left[ \frac{1}{n} \ln \frac{1}{P_{X_1^n|X_0}(X_1^n|X_0)} \right],
\end{equation}
we know $\bar{H} - \mathbb{E}[\bar{H}_{\emp}] \geq 0$.

We specify the true distribution $P_{X_0^n}(x_0^n)$ to be the i.i.d.\ product distribution $\prod_{i = 0}^n P(x_i)$, and it suffices to lower bound
\begin{align}
& \mathbb{E}_P \left[ \frac{1}{n} \ln \frac{1}{P_{X_1^n|X_0}(X_1^n|X_0)} - \min_{P\in \mathcal{M}_2(S)} \frac{1}{n} \ln \frac{1}{P_{X_1^n|X_0}(X_1^n|X_0)} \right] \\
&\quad  = H(P) - \mathbb{E}_P \left[ H(\hat{P}_{X_1 X_2}) - H(\hat{P}_{X_1}) \right] \\
& \quad = \left( H(P_{X_1 X_2}) - \mathbb{E}_P [H(\hat{P}_{X_1 X_2})] \right) - (H(P_{X_1}) - \mathbb{E}_P[H(\hat{P}_{X_1})]),
\end{align}
where $\hat{P}_{X_1 X_2}$ is the empirical distribution of the counts $\{(x_i,x_{i+1}): 0\leq i\leq n-1\}$, and $\hat{P}_{X_1}$ is the marginal distribution of $\hat{P}_{X_1 X_2}$.

It was shown in~\cite{Jiao--Venkat--Weissman2014MLE} that for any $P_{X_1}$,
\begin{align}
0 & \leq H(P_{X_1}) - \mathbb{E}_P [H(\hat{P}_{X_1})] \leq \ln \left( 1 + \frac{S-1}{n} \right).
\end{align}
Now, choosing $P_{X_1}$ to be the uniform distribution, we have
\begin{align}
& \mathbb{E}_P \left[ \frac{1}{n} \ln \frac{1}{P_{X_1^n|X_0}(X_1^n|X_0)} - \min_{P\in \mathcal{M}_2(S)} \frac{1}{n} \ln \frac{1}{P_{X_1^n|X_0}(X_1^n|X_0)} \right]  \\
& \quad \geq \ln(S^2) - \ln n -\ln \left( 1 + \frac{S-1}{n} \right) \\
& \quad \geq \ln \left( \frac{S^2}{n + S-1} \right),
\end{align}
where we have used the fact that the uniform distribution on $S$ elements has entropy $\ln(S)$, and it maximizes the entropy among all distribution supported on $S$ elements.

\section{Proof of Theorem \ref{thm.mainresultslower}}
We first show that Lemma \ref{lemma.poissontomc} and Theorem \ref{thm.lowerboundpoisson} imply Theorem \ref{thm.mainresultslower}. Firstly, Theorem \ref{thm.lowerboundpoisson} shows that as $S\to\infty$, under Poisson independent model, 
\begin{align}
\inf_{\hat{H}} \sup_{R \in \calR(S,\gamma^*,\tau,q)} \bP \left( | \hat{H} - \bar{H}| \ge C_1 \frac{S^2}{n\ln S} \right) \geq \frac{1}{2} - o(1)
\end{align}
where $\tau=S,q=\frac{1}{5\sqrt{n\ln S}}$. Moreover, since a larger $\gamma^*$ results in a smaller set of parameters for all models, we may always assume that $\gamma^*=1-C_2\sqrt{\frac{S\ln^3 S}{n}}$. For this choice of $\gamma^*$, the assumption $n\ge \frac{S^2}{\ln S}$ ensures $q=\frac{1}{5\sqrt{n\ln S}}\ge \frac{c_3\ln n}{n\gamma^*}$, and thus  Lemma \ref{lemma.poissontomc} implies
\begin{align*}
\inf_{\hat{H}} \sup_{T \in \mathcal{M}_{2,\text{rev}}(S,\gamma^*)} \bP \left( | \hat{H} - \bar{H}| \ge C_1 \frac{S^2}{n\ln S} \right) \geq \frac{1}{2} - o(1) - 2Sn^{-\frac{c_3^2}{4+10c_3}} - Sn^{-c_3/2} = \frac{1}{2} - o(1)
\end{align*}
under the Markov chain model, completing the proof of Theorem \ref{thm.mainresultslower}.

\subsection{Proof of Lemma \ref{lemma.poissontomc}}
We introduce an additional auxiliary model, namely, the \emph{independent multinomial} model, and show that the sample complexity of the Markov chain model is lower bounded by that of the independent multinomial model (Lemma~\ref{lemma.reductiontoim}), which is further lower bounded by that of the independent Poisson model (Lemma~\ref{lemma.poissontomultinomial}). To be precise, we use the notation $P_{\mathsf{MC}}$, $P_{\mathsf{IM}}$, $P_{\mathsf{IP}}$ to denote the probability measure corresponding to the three models respectively.

\subsubsection{Reduction from Markov chain to independent multinomial}

\begin{definition}[Independent multinomial model]\label{def.immodel}
	Given a stationary reversible Markov chain with transition matrix $T=(T_{ij}) \in \mathcal{M}_{2,\text{\rm rev}}(S)$, stationary distribution $\pi_i, i\in [S]$ and absolute spectral gap $\gamma^*$. Fix an integer $n\geq 0$. Under the independent multinomial model, the statistician observes $X_0\sim \pi$, and the following arrays of independent random variables
	\begin{equation*}
	\begin{matrix}
	W_{11},&  W_{12}, & \ldots, & W_{1m_1} \\
	W_{21},&  W_{22}, & \ldots, & W_{2m_2} \\
	\vdots,& \vdots,  & \ddots, & \vdots  \\
	W_{S1},&  W_{S2}, & \ldots, & W_{Sm_S}
	\end{matrix}
	\end{equation*}
	where
	the number of observations in the $i$th row is $m_i = \lceil n \pi_i  +  c_3 \max \left \{ \frac{\ln n}{\gamma^*}, \sqrt{\frac{n\pi_i \ln n}{\gamma^*}} \right\} \rceil$ for some constant $c_3\geq 20$, and within the $i$th row the random variables $W_{i1},W_{i2},\ldots,W_{im_i} \stackrel{\text{i.i.d.}}{\sim} T_{i}$.
\end{definition}

Equivalently, the observations can be summarized into the following (sufficient statistic) $S\times S$ matrix $C=(C_{ij})$, where each row is independently distributed $\multi(m_i, T_i)$, hence the name of independent multinomial model.



The following lemma relates the independent multinomial model to the Markov chain model:

\begin{lemma}\label{lemma.reductiontoim}
	If there exists an estimator $\hat{H}_1$ under the Markov chain model with parameter $n$ such that
	\begin{align}
	\sup_{T \in \mathcal{M}_{2,\text{rev}}(S, \gamma^*)} P_{\mathsf{MC}} \left( |\hat{H}_1 - \bar{H}| \geq \epsilon \right) & \leq \delta,
	\end{align}
	then there exists another estimator $\hat{H}_2$ under the independent multinomial model with parameter $n$ such that
	\begin{align}
	\sup_{T \in \mathcal{M}_{2,\text{rev}}(S, \gamma^*)}  P_{\mathsf{IM}}\left( |\hat{H}_2 - \bar{H}| \geq \epsilon \right) & \leq \delta + \frac{2S}{n^\beta},
	\end{align}
	where $\beta = \frac{c_3^2}{4 + 10c_3}\geq 1$, and $c_3$ is the constant in Definition~\ref{def.immodel}.
\end{lemma}

\subsubsection{Reduction from independent multinomial to independent Poisson}
For the reduction from the independent multinomial model to the independent Poisson model, we have the following lemma. Note that 
\begin{align}
\bar{H}(T(R)) & = \sum_{1\leq i,j\leq S} \pi_i \frac{R_{ij}}{\sum_{j = 1}^S R_{ij}} \ln \frac{\sum_{j = 1}^S R_{ij}}{R_{ij}} \\
& = \frac{1}{r} \sum_{1\leq i,j\leq S} R_{ij} \ln \frac{r_i}{R_{ij}} \\
& = \frac{1}{r} \left( \sum_{1\leq i,j\leq S} R_{ij}\ln \frac{1}{R_{ij}} + \sum_{i = 1}^S r_i \ln r_i \right).
\label{eq:HTR}
\end{align}
\begin{lemma}\label{lemma.poissontomultinomial}
	If there exists an estimator $\hat{H}_1$ for the independent multinomial model with parameter $n$ such that
	\begin{align}
	\sup_{T \in \mathcal{M}_{2,\text{rev}}(S, \gamma)} P_{\mathsf{IM}} \left( |\hat{H}_1- \bar{H}| \geq \epsilon \right) & \leq \delta,
	\label{eq:imguarantee1}
	\end{align}
	then there exists another estimator $\hat{H}_2$ for the independent Poisson model with parameter $\lambda=\frac{4 n}{\tau}$ such that
	\begin{align}
	\sup_{R \in \calR(S,\gamma,\tau,q) } P_{\mathsf{IP}}\left( |\hat{H}_2 - \bar{H}(T(R))| \geq  \epsilon \right) & \leq  \delta +  S n^{-c_3/2},
	\label{eq:ipguarantee1}
	\end{align}
	provided $q \geq \frac{ c_3 \ln n}{n \gamma}$, where $c_3\geq 20$ is the constant in Definition~\ref{def.immodel}.
\end{lemma}

\subsection{Proof of Theorem \ref{thm.lowerboundpoisson}}
Now our task is reduced to lower bounding the sample complexity of the independent Poisson model. The general strategy is the so-called method of fuzzy hypotheses, which is an extension of LeCam's two-point methods.
The following version is adapted from \cite[Theorem 2.15]{Tsybakov2008} (see also \cite[Lemma 11]{HJWW17}).

\begin{lemma} \label{lemma.tsybakov}
	Let ${\bf Z}$ be a random variable distributed according to $P_\theta$ for some $\theta \in \Theta$.
	Let $\mu_1,\mu_2$ be a pair of probability measures (not necessarily supported on $\Theta$).
	Let $\hat{f} = \hat{f}({\bf Z})$ be an arbitrary estimator of the functional $f(\theta)$ based on the observation $\bf Z$.
	Suppose there exist $\zeta\in \mathbb{R}, \Delta>0, 0\leq \beta_1,\beta_2 <1$ such that
	\begin{align}
	\mu_1(\theta \in \Theta: f(\theta) \leq \zeta -\Delta) & \geq 1-\beta_1 \\
	\mu_2(\theta \in \Theta: f(\theta) \geq \zeta + \Delta) & \geq 1-\beta_2.
	\end{align}
	Then
	\begin{equation}
	\inf_{\hat{f}} \sup_{\theta \in \Theta} \bP_\theta\left( |\hat{f} - f(\theta)| \geq \Delta \right) \geq \frac{1- \mathsf{TV}(F_1,F_2)- \beta_1 - \beta_2}{2},
	\end{equation}
	where $F_i=\int P_\theta \mu_i(d\theta)$ is the marginal distributions of $\mathbf{Z}$ induced by the prior $\mu_i$, for $i = 1,2$,
	and $\mathsf{TV}(F_1,F_2) = \frac{1}{2}\int |dF_1 - dF_2|$ is the total variation distance between distributions $F_1$ and $F_2$.
\end{lemma}
To apply this method for the independent Poisson model,
the parameter is the $S\times S$ symmetric matrix $R$, the function to be estimated is $\bar H = \bar H(T(R))$, the observation (sufficient statistic for $R$) is
\[
\bfC = X_0\cup \{ C_{ij} + C_{ji}: i\neq j, 1\leq i\leq j\leq S\} \cup \{C_{ii}: 1\leq i\leq S\}.
\]
The goal is to construct two symmetric random matrices (whose distributions serve as the priors), such that
\begin{enumerate}
	\item[(a)] they are sufficiently concentrated near the desired parameter space $\calR(S,\gamma,\tau,q)$ for properly chosen parameters $\gamma,\tau,q$;
	\item[(b)] the entropy rates have different values;
	\item[(c)] the induced marginal laws of $\bfC$ are statistically inseparable.
\end{enumerate}
To this end, we need the following results (cf.~\cite[Proof of Proposition 3]{Wu--Yang2014minimax}):

\begin{lemma}\label{lemma.wuyangprior}
	Let
	\[
	\phi(x) \triangleq x\ln \frac{1}{x}, \quad x \in[0,1].
	\]
	Let $c>0, D>100$ and $0<\eta_0<1$ be some absolute constants.
	For any $\alpha \in (0,1), \eta\in (0,\eta_0)$, there exist random variables $U, U'$ supported on $[0, \alpha \eta^{-1}]$ such that
	\begin{align}
	\mathbb{E}[\phi(U)] - \mathbb{E}[\phi(U')] & \geq c \alpha \\
	\mathbb{E}[U^j] & = \mathbb{E}[U'^j], \quad j = 1,2,\ldots, \left \lceil \frac{D}{\sqrt{\eta}} \right \rceil \\
	\mathbb{E}[U] & = \mathbb{E}[U'] = \alpha.
	\end{align}
\end{lemma}



\begin{lemma}[{\cite[Lemma 3]{Wu--Yang2014minimax}}]
	\label{lemma.poissontv}
	Let $ V_1 $ and $ V_2 $ be random variables taking values in $ [0,M] $.     If $ \Expect[V_1^j]=\Expect[V_2^j],~j=1,\dots,L $,
	then
	\begin{equation}
	\TV(\Expect[\Poi(V_1)],\Expect[\Poi(V_2)]) \le \pth{\frac{2eM}{L}}^L.
	\label{eq:tv-bound}
	\end{equation}
	where $\mathbb{E}[\mathsf{Poi}(V)] = \int \Poi(\lambda) P_V(d\lambda)$ denotes the Poisson mixture with respect to the distribution of a positive random variable $V$.
\end{lemma}

Now we are ready to define the priors $\mu_1, \mu_2$ for the independent Poisson model. For simplicity,
we assume the cardinality of the state space is $S+1$ and introduce a new state $0$:
\begin{definition}[Prior construction]\label{con.priorconstruction}
	Suppose $n\geq \frac{S^2}{\ln S}$. Set
	\begin{align}
	\alpha & = \frac{S}{n\ln S} \leq \frac{1}{S} \\
	\frac{1}{\eta} & = (d_1 \ln S)^2 \\
	L & = \ceil{\frac{D}{\sqrt{\eta}} }, \label{eqn.etadefinition}
	\end{align}
	where $d_1 = \frac{D}{8e^2}$, and $D>0$ is the constant in Lemma~\ref{lemma.wuyangprior}.
	
	Recall the random variables $U,U'$ are introduced in Lemma~\ref{lemma.wuyangprior}.
	We use a construction that is akin to that studied in \cite{bordenave2010spectrum}.
	Define $S\times S$ symmetric random matrices $\bfU=(U_{ij})$ and $\bfU'=(U_{ij}')$, where $\{U_{ij}: 1 \leq i \leq j \leq S\} $ be i.i.d.~copies of $U$ and
	$\{U'_{ij}: 1 \leq i \leq j \leq S\} $ be i.i.d.~copies of $U'$, respectively.
	Let
	\begin{equation}
	\bfR=
	\left[
	\begin{array}{c|c}
	b & a \cdots a \\ \hline
	a & \raisebox{-15pt}{\mbox{{$\bfU$}}} \\[-4ex]
	\vdots & \\[-0.5ex]
	a &
	\end{array}
	\right], \qquad
	\bfR'=
	\left[
	\begin{array}{c|c}
	b & a \cdots a \\ \hline
	a & \raisebox{-15pt}{\mbox{{$\bfU'$}}} \\[-4ex]
	\vdots & \\[-0.5ex]
	a &
	\end{array}
	\right],
	\label{eq:R}
	\end{equation}
	where
	\begin{equation}
	a = \sqrt{\alpha S}, \quad b = S.
	\label{eq:ab}
	\end{equation}
	Let $\mu_1$ and $\mu_2$ be the laws of $\bfR$ and $\bfR'$, respectively. The parameters $\gamma,\tau,q$ will be chosen later, and we set $\lambda=\frac{4n}{\tau}$ in the independent Poisson model (as in Lemma \ref{lemma.poissontomultinomial}).
\end{definition}

The construction of this pair of priors achieves the following three goals:
\paragraph{(a) Statistical indistinguishablility.}
Note that the distributions of the first row and column of $\bfR$ and $\bfR'$ are identical. Hence the
sufficient statistics are $X_0$ and $\bfC = \{ C_{ij} + C_{ji}: i\neq j, 1\leq i\leq j\leq S\} \cup \{C_{ii}, 1\leq i\leq S\}$.
Denote its the marginal distribution as $F_i$ under the prior $\mu_i$, for $i=1,2$.
The following lemma shows that the distributions of the sufficient statistic are indistinguishable:
\begin{lemma}\label{lemma.indistinguishableipmodel}
	For $n\ge \frac{S^2}{\ln S}$, we have $\mathsf{TV}(F_1,F_2) = o(1)$  as $S\to\infty$.
\end{lemma}

\paragraph{(b) Functional value separation.} Under the two priors $\mu_1,\mu_2$, the corresponding entropy rates of the independent Poisson model differ by a constant factor of $\frac{S^2}{n\ln S}$. Here we explain the intuition: in view of \eqref{eq:HTR}, for $\phi(x)=-x\ln x$ we have
\begin{align}
\bar H(T(\bfR)) = & ~ \frac{1}{r} \left( \sum_{i,j=0}^{S} \phi(R_{ij}) - \sum_{i = 0}^S \phi(r_i) \right)	\label{eq:HTR2}
\end{align}
where $r_i=\sum_{j = 0}^S R_{ij}$ and $r = \sum_{i,j = 0}^S R_{ij}$; similarly,
\begin{align}
\bar H(T(\bfR')) = & ~ \frac{1}{r'} \left( \sum_{i,j=0}^{S} \phi(R_{ij}') - \sum_{i = 0}^S \phi(r_i') \right).	\nonumber
\end{align}
We will show that both $r$ and $r'$ are close to their common mean $b + 2aS + S^2 \alpha = S(1+\sqrt{\alpha S})^2 \approx S$.
Furthermore, $r_i$ and $r_i'$ also concentrate on their common mean. Thus, in view of Lemma~\ref{lemma.wuyangprior}, we have
\begin{align}
|\bar H(T(\bfR)) - \bar H(T(\bfR'))|  & \approx S |\mathbb{E}[\phi(U)] - \mathbb{E}[\phi(U')]| = \Omega(S \alpha) = \Omega\pth{\frac{S^2}{n \ln S}}.
\end{align}
The precise statement is summarized in the following lemma:
\begin{lemma}\label{lemma.functionalseperation}
	Assume that $n \geq \frac{S^2}{\ln S}$ and $\ln n \ll \frac{S}{\ln^2 S}$. There exist universal constants $C_1>0$ and some $\zeta \in \reals$, such that as $S\to \infty$,
	\begin{align}
	\prob{\bar H(T(\bfR)) \geq \zeta + C_1 \frac{S^2}{n\ln S}  } &  = 1 - o(1),	\nonumber \\
	\prob{\bar H(T(\bfR')) \leq \zeta -  C_1 \frac{S^2}{n\ln S} } &  = 1 - o(1).	\nonumber
	\end{align}
\end{lemma}

\paragraph{(c) Concentration on parameter space.} Although the random matrices $\bfR$ and $\bfR'$ may take values outside the desired space $\calR\left( S, \gamma, \tau,q\right)$, we show that most of the mass is concentrated on this set with appropriately chosen parameters. The following lemma, which is the core argument of the lower bound, makes this statement precise.

\begin{lemma}\label{lemma.spectralgapcontrol}
	Assume that $n \geq \frac{S^2}{\ln S}$.
	There exist universal constants $C>0$, such that as $S\to \infty$,
	\begin{align}
	\prob{\bfR \in \calR\left( S, \gamma,\tau,q\right)  } &  = 1 - o(1),	\nonumber \\
	\prob{\bfR' \in \calR\left( S, \gamma,\tau,q\right) } &  = 1 - o(1),	\nonumber
	\end{align}
	where $\gamma = 1 - C_2 \sqrt{\frac{S \ln^3 S}{n}}$, $\tau = S$, and $q = \frac{1}{5\sqrt{n \ln S}}$.
\end{lemma}

Fitting Lemma \ref{lemma.indistinguishableipmodel}, Lemma \ref{lemma.functionalseperation} and Lemma \ref{lemma.spectralgapcontrol} into the main Lemma \ref{lemma.tsybakov}, the following minimax lower bound holds for the independent Poisson model.
\begin{proof}[Proof of Theorem \ref{thm.lowerboundpoisson}]
	For the choice of $\zeta$ and $\Delta=C_1\frac{S^2}{n\ln S}$ in Lemma \ref{lemma.functionalseperation}, a combination of Lemma \ref{lemma.functionalseperation} and Lemma \ref{lemma.spectralgapcontrol} gives
	\begin{align}
	\prob{ \bfR \in \calR\left( S, \gamma^*,\tau,q\right), \bar H(T(\bfR)) \geq \zeta + C_1 \frac{S^2}{n\ln S} } = 1-o(1)
	\end{align}
	as $S\to\infty$, so that $\beta_1=o(1)$. Similarly, $\beta_2=o(1)$. By Lemma \ref{lemma.indistinguishableipmodel}, we have $\mathsf{TV}(F_1,F_2)=o(1)$. Now Theorem \ref{thm.lowerboundpoisson} follows from Lemma \ref{lemma.tsybakov} directly.
\end{proof}

\section{Experiments}\label{sec.experiments}

The entropy rate estimator we proposed in this paper that achieves the minimax rates can be viewed as a \emph{conditional} approach; in other words, we apply a Shannon entropy estimator for observations corresponding to each state, and then average the estimates using the empirical frequency of the states. More generally, for any estimator $\hat{H}$ of the Shannon entropy from i.i.d.~data, the conditional approach follows the idea of
\begin{align}
\bar{H}_{\text{Cond}} & = \sum_{i = 1}^S \hat{\pi}_i \hat{H}(\mathbf{X}^{(i)}),
\end{align}
where $\hat{\pi}$ is the empirical marginal distribution. We list several choices of $\hat{H}$:
\begin{enumerate}
	\item The empirical entropy estimator, which simply evaluates the Shannon entropy of the empirical distribution of the input sequence. It was shown not to achieve the minimax rates in Shannon entropy estimation~\cite{Jiao--Venkat--Weissman2014MLE}, and also not to achieve the optimal sample complexity in estimating the entropy rate in Theorem~\ref{thm.biasmlebound} and Corollary~\ref{cor.empirical}.
	\item The Jiao--Venkat--Han--Weissman (JVHW) estimator, which is based on best polynomial approximation and proved to be minimax rate-optimal in~\cite{Jiao--Venkat--Han--Weissman2015minimax}. The independent work~\cite{Wu--Yang2014minimax} is based on similar ideas.
	\item The Valiant--Valiant (VV) estimator, which is based on linear programming and proved to achieve the $\frac{S}{\ln S}$ phase transition for Shannon entropy in~\cite{Valiant--Valiant2013estimating}. 
	\item The profile maximum likelihood estimator (PML), which is proved to achieve the $\frac{S}{\ln S}$ phase transition in~\cite{acharya2016unified}. However, there does not exist an efficient algorithm to even approximately compute the PML with provably ganrantees. 
\end{enumerate}

There is another estimator, \emph{i.e.}, the Lempel--Ziv (LZ) entropy rate estimator \cite{Ziv--Lempel1978compression}, which does not lie in the category of conditional approaches. The LZ estimator estimates the entropy through compression: it is well known that for a universal lossless compression scheme, its codelength per symbol would approach the Shannon entropy rate as length of the sample path grows to infinity. Specifically, for the following random matching length defined by
\begin{align}
L_{i}^{n}=1+\max \left\{ 1\le l\le n:\exists j\le i-1\ s.t.\ \left( X_i,\cdots ,X_{i+l-1} \right) =\left( X_j,\cdots ,X_{j+l-1} \right) \right\},
\end{align}
it is shown in~\cite{Wyner--Ziv1989} that for stationary and ergodic Markov chains,
\begin{align}
\lim_{n\to\infty} \frac{L_{i}^{n}}{\ln n}=\bar{H}\ a.s.
\end{align}

We use alphabet size $S=200$ and vary the sample size $n$ from $100$ to $300000$
to demonstrate how the performance varies as the sample size increases. We compare the performance of the estimators by measuring the root mean square error (RMSE) in the following four different scenarios via $10$ Monte Carlo simulations:

\begin{enumerate}
	\item Uniform: The eigenvalue of the transition matrix is uniformly distributed except the largest one and the transition matrix is generated using the method in~\cite{jiang2009construction}. Here we use spectral gap $\gamma=0.1$.
	\item Zipf: The transition probability $T_{ij}\propto \frac{1}{i+j}$.
	\item Geometric: The transition probability $T_{ij}\propto 2^{-\left| i-j \right|}$.
	\item Memoryless: The transition matrix consists of identical rows.
\end{enumerate}

In all of the four cases, the JVHW estimator outperforms the empirical entropy rate. The results of VV~\cite{Valiant--Valiant2013estimating} and LZ~\cite{Wyner--Ziv1989} are not included due to their considerable longer running time. For example, when $S=200$ and $n=300000$ and we try to estimate the entropy rate from a single trajectory of the Markov chain, the empirical entropy and the JVHW estimator were evaluated in less than 30 seconds.  The evaluation of LZ estimator and the conditional VV method did not terminate after a month.\footnote{For LZ, we use the Matlab implementation in \url{https://www.mathworks.com/matlabcentral/fileexchange/51042-entropy-estimator-based-on-the-lempel-zivalgorithm?focused=3881655&tab=function}. For VV, we use the Matlab implementation in \url{http://theory.stanford.edu/~valiant/code.html}. We use 10 cores of a server with CPU frequency 1.9GHz.}
The main reason for the slowness of the VV methods in the context of Markov chains is that for each context it needs to call the original VV entropy estimator ($2000$ times in total in the above experiment), each of which needs to solve a linear programming.

\begin{figure}[htbp]
	\centering
	\subfigure[Uniform]{
		\includegraphics[width={0.45\columnwidth}]{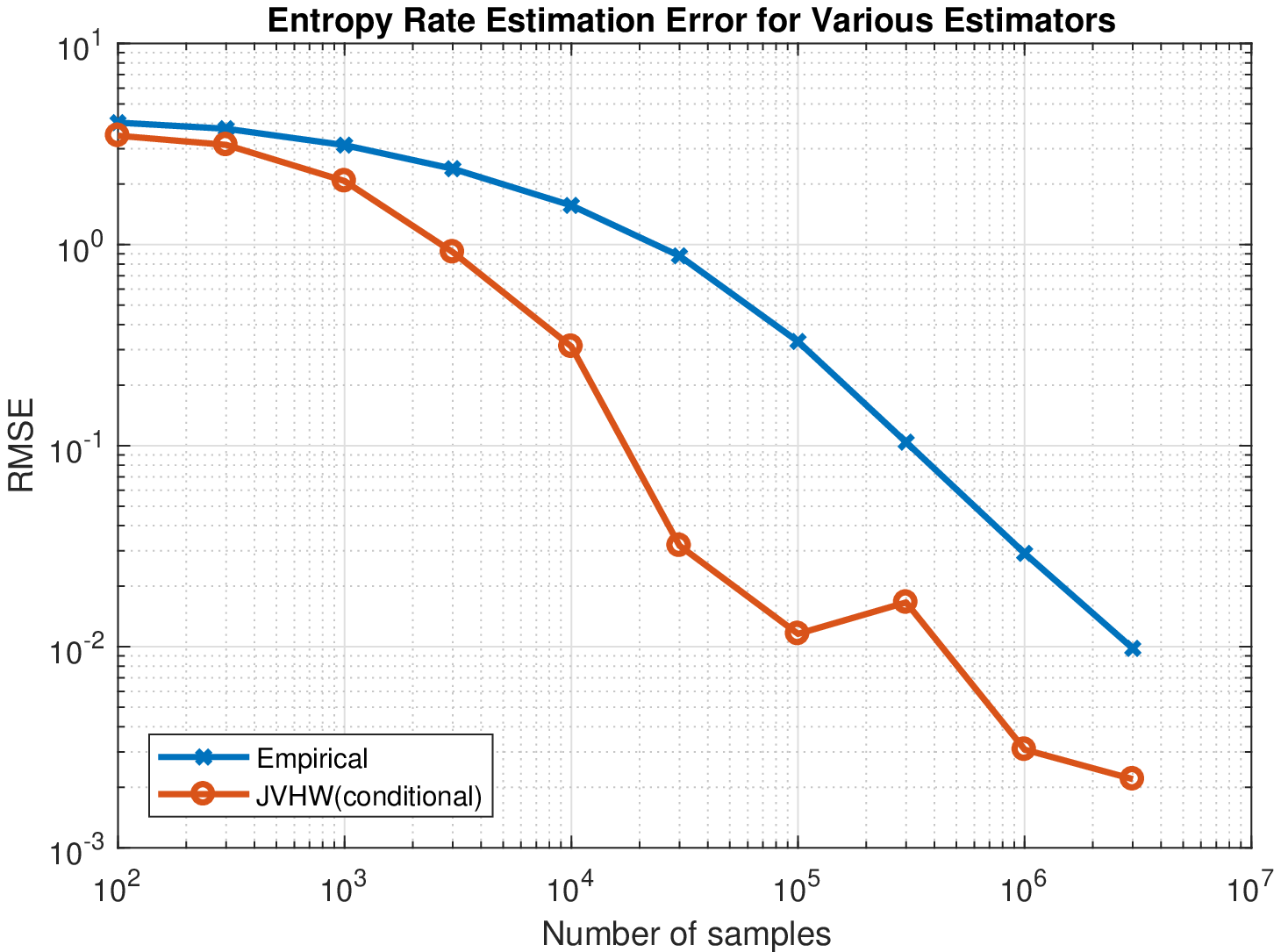}
	}
	\subfigure[Zipf]{
		\includegraphics[width={0.45\columnwidth}]{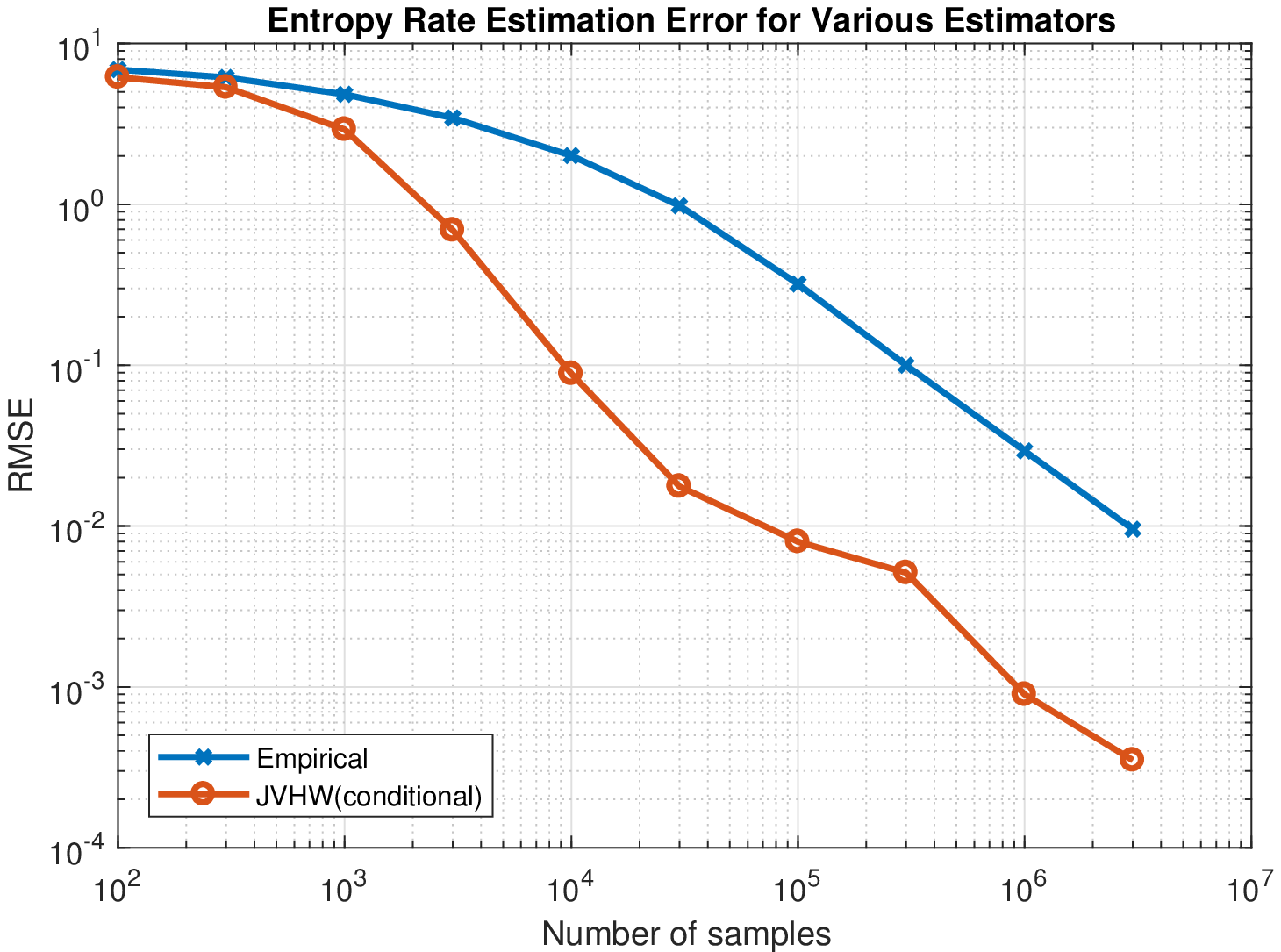}
	}
	
	\subfigure[Geometric]{
		\includegraphics[width={0.45\columnwidth}]{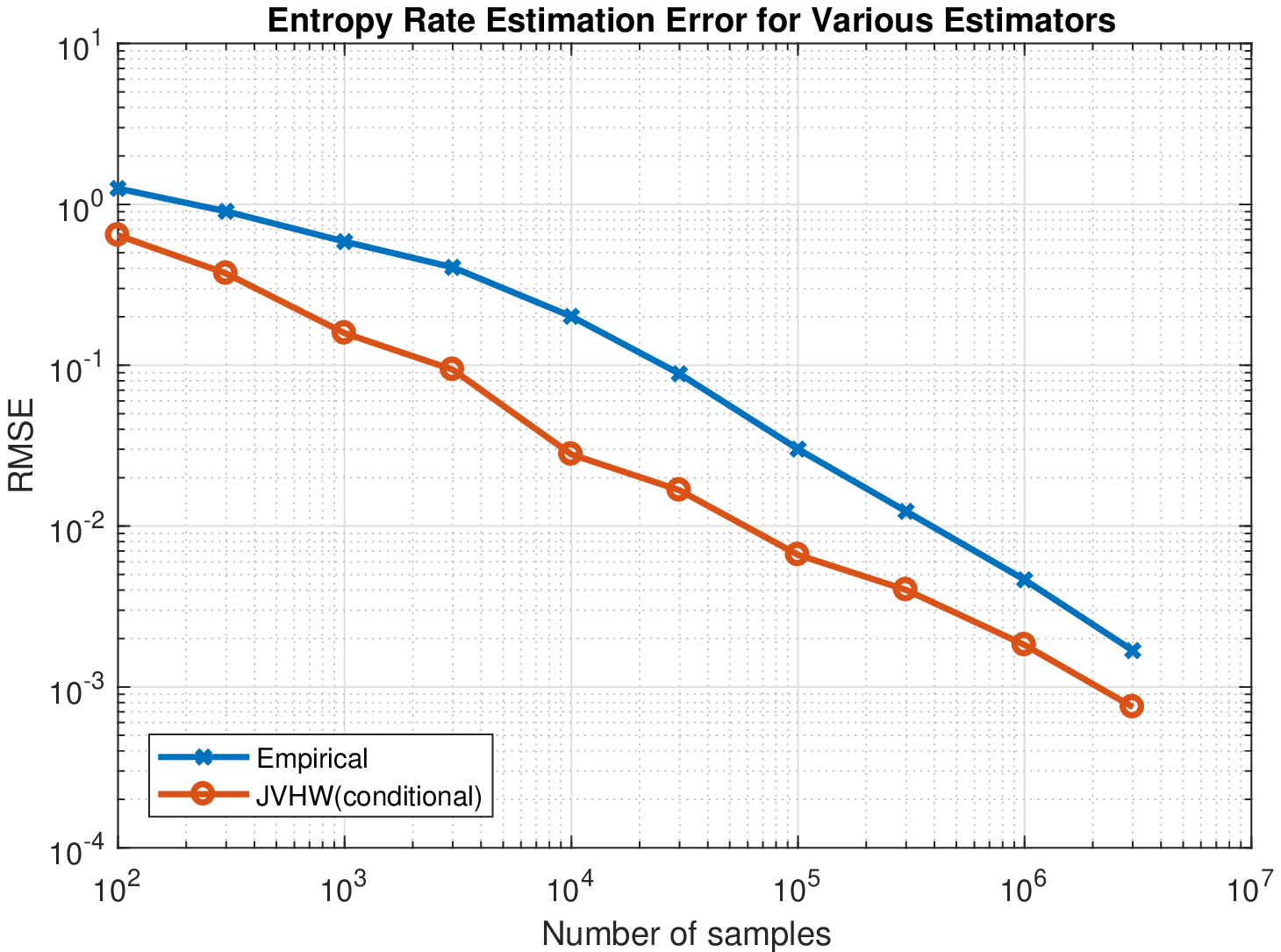}
	}
	\subfigure[Memoryless]{
		\includegraphics[width={0.45\columnwidth}]{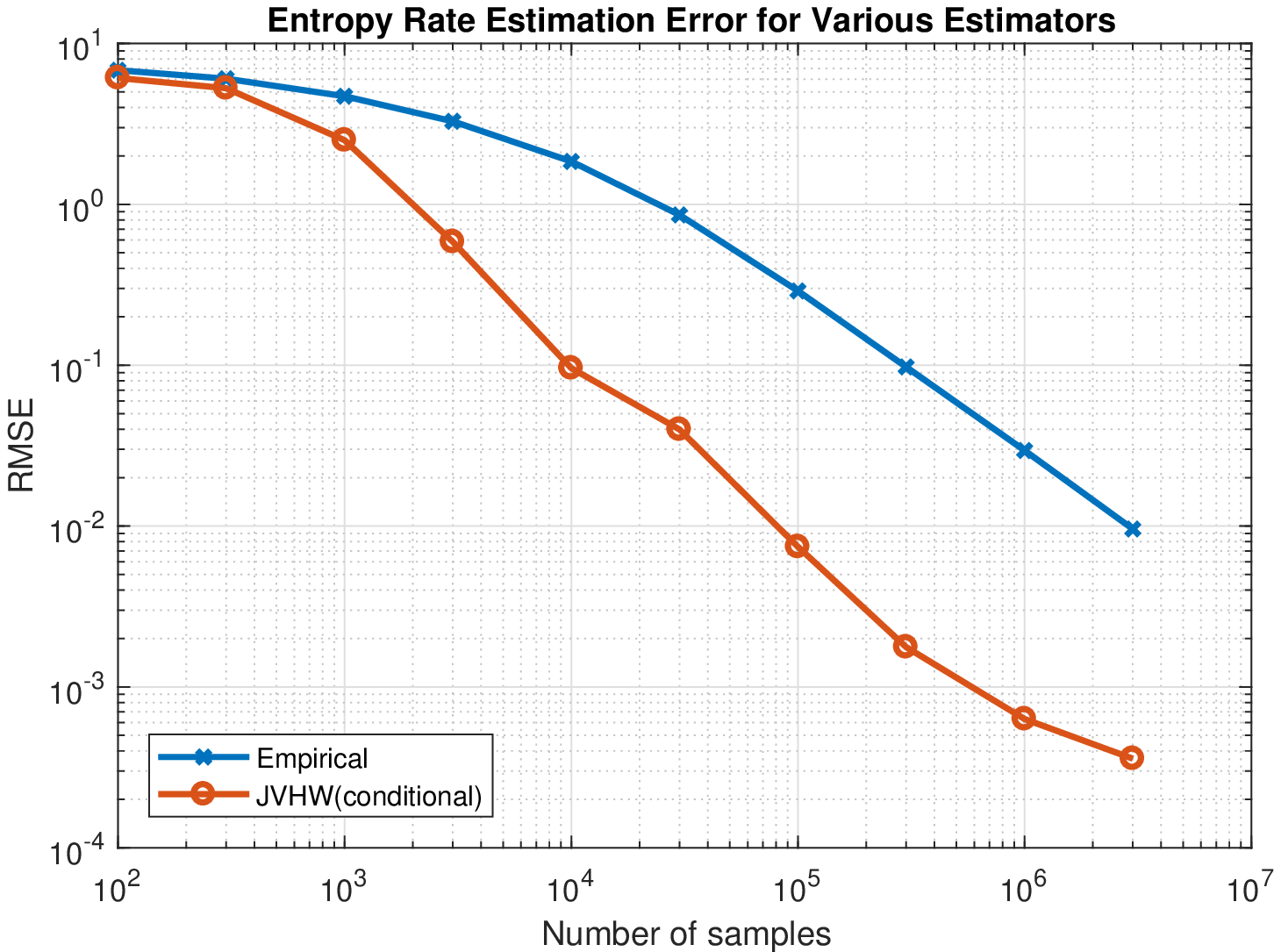}
	}
	\caption{Comparison of the performances of the empirical entropy rate and JVHW estimator in different parameter configurations}
\end{figure}


\section{More on fundamental limits of language modeling}\label{sec.languagemodels_appendix}

Since the number of words in the English language (\textit{i.e.}, our ``alphabet'' size) is huge, in view of the $\frac{S^2}{\log S}$ result we showed in theory, a natural question is whether a corpus as vast as the 1BW corpus is enough to allow reliable estimates of conditional entropy (as in Figure \ref{fig:language-entropy}). A quick answer to this question is that our theory has so far focused on the worst-case analysis and, as demonstrated below, natural language data are much nicer so that the sample complexity for accurate estimation is much lower than what the minimax theory predicts.
Specifically, we computed the conditional entropy estimates of Figure~\ref{fig:language-entropy} but this time restricting the sample to only a subset of the corpus. A plot of the resulting estimate as a function of sample size is shown in Figures~\ref{fig:language-sample-size-1} and \ref{fig:language-sample-size-23}. Because sentences in the corpus are in randomized order, the subset of the corpus taken is randomly chosen.

To interpret these results, first, note the number of distinct unigrams (\textit{i.e.}, words) in the 1BW corpus is about two million.  We recall that in the i.i.d.\ case, $n \gg S/\ln S$ samples are necessary \cite{Valiant--Valiant2011,Wu--Yang2014minimax,Jiao--Venkat--Han--Weissman2015minimax}, even in the worst case a dataset of 800 million words will be more than adequate to provide a reliable estimate of entropy for $S \approx 2$~million. Indeed, the plot for unigrams with the JVHW estimator in Figure~\ref{fig:language-sample-size-1} supports this.  In this case, the entropy estimates for all sample sizes greater than 338\,000 words is within 0.1 bits of the entropy estimate using the entire corpus. That is, it takes just 0.04\% of the corpus to reach an estimate within 0.1 bits of the true value.

We note also that the empirical entropy rate converges to the same value, 10.85, within two decimal places. This is also shown in Figure~\ref{fig:language-sample-size-1}. The dotted lines indicate the final entropy estimate (of each estimator) using the entire corpus of $7.7\times 10^8$ words.

Results for similar experiments with bigrams and trigrams are shown in Figure~\ref{fig:language-sample-size-23} and Table~\ref{tab:language-convergence}. Since the state space for bigrams and trigrams is much larger, convergence is naturally slower, but it nonetheless appears fast enough that our entropy estimate should be within on the order of 0.1 bits of the true value.

\begin{figure}
	\centering
	\includegraphics{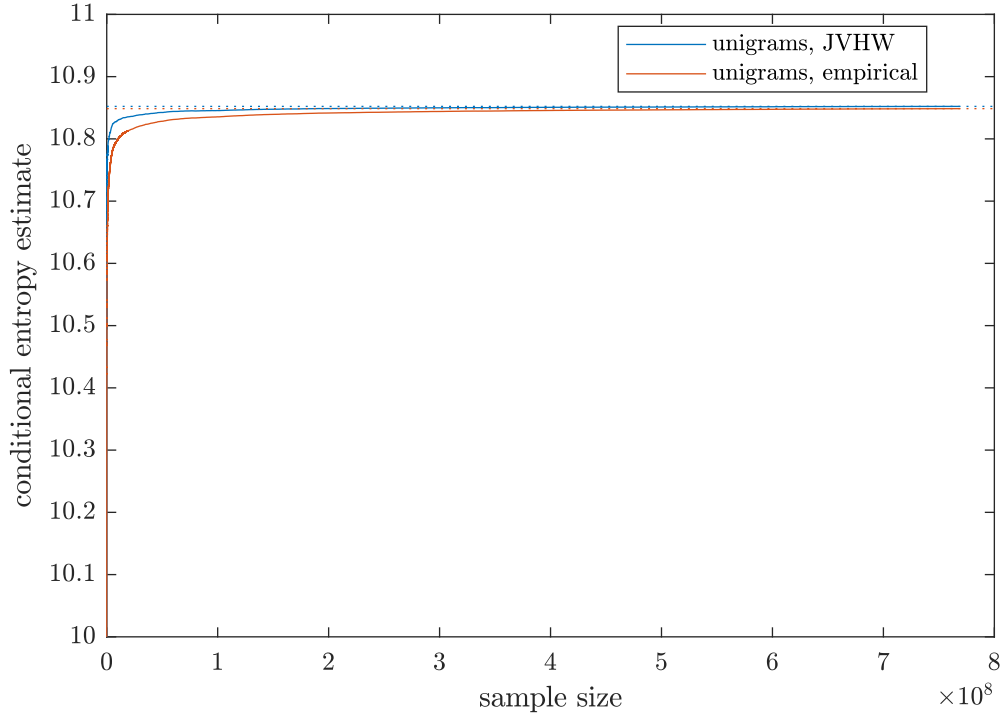}
	\caption{Estimates of conditional entropy versus sample size for 1BW unigrams; dotted lines are the estimate using the entire corpus (\textit{i.e.}, the final estimate). Note the zoomed-in axes.}
	\label{fig:language-sample-size-1}
\end{figure}

\begin{figure}
	\centering
	\includegraphics{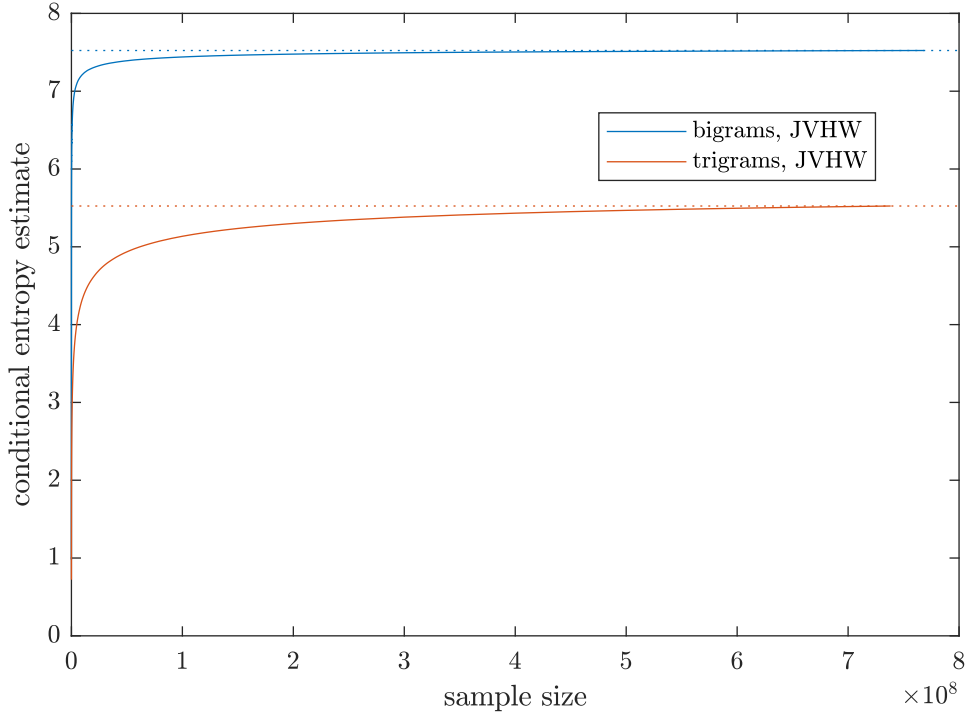}
	\caption{Estimates of conditional entropy versus sample size for 1BW bigrams and trigrams; dotted lines are the estimate using the entire corpus (\textit{i.e.}, the final estimate)}
	\label{fig:language-sample-size-23}
\end{figure}

 \begin{table}
     \centering
     \caption{Convergence points for 1BW conditional entropy estimates (within 0.1 bit of final estimate)}
     \label{tab:language-convergence}
     \begin{tabular}{r||r|r||r|r}
     & \multicolumn{2}{c||}{JVHW estimator} & \multicolumn{2}{c}{empirical entropy} \\
     $k$ & sample size & \% of corpus & sample size & \% of corpus \\
     \hline
     1 & 338k & 0.04\% & 2.6M & 0.34\% \\
     2 &  77M & 10.0\% & 230M & 29.9\% \\
     3 & 400M & 54.2\% & 550M & 74.5\% \\
     \end{tabular}
 \end{table}

\begin{table}
	\centering
	\caption{Points at which the 1BW entropy estimates are within 0.1 bit of the final estimate}
	\label{tab:language-convergence}
	\begin{tabular}{r||r|r}
		$k$ & sample size & \% of corpus \\
		\hline
		1 & 338k & 0.04\% \\
		2 &  77M & 10.0\% \\
		3 & 400M & 54.2\% \\
	\end{tabular}
\end{table}

With these observations, we believe that the estimates based on the 1BW corpus should have enough samples to produce reasonably reliable entropy estimates. As one further measure, to approximate the variance of these entropy estimates, we also ran bootstraps for each memory length $k = 1, \dots, 4$, with a bootstrap size of the same size as the original dataset (sampling with replacement). For the 1BW corpus, with 100 bootstraps, the range of estimates (highest less lowest) for each memory length never exceeded 0.001 bit, and the standard deviation of estimates was just 0.0002---that is, the error ranges implied by the bootstraps are too small to show legibly on Figure~\ref{fig:language-entropy}. For the PTB corpus, also with 100 bootstraps, the range never exceeded 0.03 bit. Further details of our bootstrap estimates are given in Table~\ref{tab:language-bootstrap}.

\begin{table}
	\centering
	\caption{Bootstrap estimates of error range}
	\label{tab:language-bootstrap}
	\begin{tabular}{r||r|r|r||r|r|r}
		& \multicolumn{3}{|c||}{PTB} & \multicolumn{3}{|c}{1BW} \\
		$k$ & estimate & st.\ dev. & range & estimate & st. dev. & range \\
		\hline
		1 & 10.62 & 0.00360 & 0.0172 & 10.85 & 0.000201 & 0.00091 \\
		2 &  6.68 & 0.00360 & 0.0183 &  7.52 & 0.000152 & 0.00081 \\
		3 &  3.44 & 0.00384 & 0.0159 &  5.52 & 0.000149 & 0.00078 \\
		4 &  1.50 & 0.00251 & 0.0121 &  3.46 & 0.000173 & 0.00081
	\end{tabular}
\end{table}

\section{Auxiliary lemmas} \label{sec.auxiliarylemmas}

\begin{lemma}\label{lemma.representationofempiricalentropy}
	For an arbitrary sequence $(x_0,x_1,\ldots,x_n) \in \mathcal{X}^{n+1}, \mathcal{X} = \{1,2,\ldots,S\}$,
	define the empirical distribution of the consecutive pairs as
	$\hat{P}_{X_1 X_2} = \frac{1}{n} \sum_{i=0}^{n-1} \delta_{(x_i,x_{i+1})}$.
	Let 	$\hat{P}_{X_1} = \frac{1}{n} \sum_{i=0}^{n-1} \delta_{x_i}$ be the marginal distribution of $\hat{P}_{X_1 X_2}$, and the empirical frequency of state $i$ as
	\begin{align}
	\hat{\pi}_i = \frac{1}{n} \sum_{j = 0}^{n-1} \mathbbm{1}(x_j = i).
	\end{align}
	Denote the empirical conditional distribution as $\hat P_{X_2|X_1} = \frac{\hat P_{X1X_2}}{\hat P_{X_1}}$, \emph{i.e.},
	\begin{align}
	\hat P_{X_2|X_1=i}(j) & = \frac{\sum_{m = 1}^{n} \mathbbm{1}(x_m = j, x_{m-1} = i)}{  n \hat{\pi}_i},
	\end{align}
	whenever $\hat{\pi}_i$.
	Let $\Hemp = \sum_{i = 1}^S \hat{\pi}_i H(\hat P_{X_2|X_1=i})$
	and $H(\cdot)$ is the Shannon entropy. 
	Then, we have
	\begin{align}
	\Hemp
	& = H(\hat{P}_{X_1 X_2}) -H(\hat{P}_{X_1}), \label{eq:Hemp2}\\
	& = \min_{P \in \mathcal{M}_2(S)} \frac{1}{n} \ln \frac{1}{P_{X_1^n|X_0}(x_1^n|x_0)} \label{eq:Hemp1}
	\end{align}
	where in \prettyref{eq:Hemp1},
	for a given transition matrix $P$, $P_{X_1^n|X_0}(x_1^n|x_0) \triangleq \prod_{t=0}^{n-1} P(x_{t+1},x_t)$.
\end{lemma}

The following lemma gives well-known tail bounds for Poisson and Binomial random variables.
\begin{lemma}\cite[Exercise 4.7]{mitzenmacher2005probability}\label{lemma.poissontail}
	If $X\sim \spo(\lambda)$ or $X\sim \mathsf{B}(n,\frac{\lambda}{n})$, then for any $\delta>0$, we have
	\begin{align}
	\bP(X \geq (1+\delta) \lambda) & \leq \left( \frac{e^\delta}{(1+\delta)^{1+\delta}} \right)^\lambda \le e^{-\delta^2\lambda/3}\vee e^{-\delta\lambda/3} \\
	\bP(X \leq (1-\delta)\lambda) & \leq  \left( \frac{e^{-\delta}}{(1-\delta)^{1-\delta}} \right)^\lambda \leq e^{-\delta^2 \lambda/2}.
	\end{align}
\end{lemma}

The following lemma is the Hoeffding inequality.
\begin{lemma}\label{lem_hoeffding}
	\cite{Hoeffding1963probability}
	Let $X_1,X_2,\ldots,X_n$ be independent random variables such that $X_i$ takes its value in $[a_i,b_i]$ almost surely for all $i\leq n$. Let $S_n=\sum_{i=1}^n X_i$, we have for any $t>0$,
	\begin{align}
	P\left\{|S_n-\bE[S_n]|\ge t\right\} \le 2\exp\left(-\frac{2t^2}{\sum_{i = 1}^n (b_i-a_i)^2}\right).
	\end{align}
\end{lemma}

\section{Proofs of main lemmas}\label{sec.mainlemmasproof}

\subsection{Proof of Lemma~\ref{lemma.goodeventshighprobability}}

We being with a lemma on the concentration of the empirical distribution $\hat{\pi}$ for reversible Markov chains.
\begin{lemma}\label{lemma.confidencepi}
Consider a reversible stationary Markov chain with spectral gap $\gamma$. Then, for every $i, 1\leq i\leq S$, every constant $c_3>0$, the event
\begin{align}
\label{eq:confidencepi}
\mathcal{E}_i & =  \left \{ |\hat{\pi}_i - \pi_i| \geq c_3 \max \left \{ \frac{\ln n}{n\gamma}, \sqrt{\frac{\pi_i \ln n}{n \gamma}} \right \} \right \}
\end{align}
happens with probability at most $\frac{2}{n^\beta}$, where $\beta = \frac{c_3^2}{4 + 10 c_3}$.
\end{lemma}

\begin{proof}[Proof of Lemma~\ref{lemma.confidencepi}]
Recall the following Bernstein inequality for reversible chains \cite[Theorem 3.3]{Paulin15}:
For any stationary reversible Markov chain with spectral gap $\gamma$,
\begin{align}
P\sth{ |\hat{\pi}_i - \pi_i| \geq \frac{t}{n}} & \leq  2 \exp\left(  - \frac{t^2 \gamma}{4n \pi_i (1-\pi_i) + 10t} \right).
\label{eq:paulin}
\end{align}
We have $\frac{\ln n}{n\gamma} \geq \sqrt{\frac{\pi_i \ln n}{n\gamma}}$ if and only if $\pi_i \leq \frac{\ln n}{n\gamma}$. We split the proof of \prettyref{eq:confidencepi} into two parts.
\begin{enumerate}
\item $\pi_i \leq \frac{\ln n}{n\gamma}$: Invoking \prettyref{eq:paulin} and setting $t = c_3 \frac{\ln n}{\gamma}$, we have
\begin{align*}
P_\pi \left( |\hat{\pi}_i - \pi_i| \geq c_3 \frac{\ln n}{n\gamma} \right ) & \leq 2 \exp \left( - \frac{\gamma c_3^2 \frac{\ln^2 n}{\gamma^2}}{4 n  \frac{\ln n}{n\gamma} + 10 c_3 \frac{\ln n}{\gamma}} \right) \\
& \leq 2 \exp \left( - \frac{c_3^2}{4 + 10c_3} \ln n \right) \\
& = \frac{2}{n^\beta}.
\end{align*}
\item $\pi_i \geq \frac{\ln n}{n\gamma}$: Invoking \prettyref{eq:paulin} and setting $t = c_3 \sqrt{ \frac{n \pi_i \ln n}{\gamma} }$, we have
\begin{align*}
P_\pi \left( |\hat{\pi}_i - \pi_i| \geq c_3 \sqrt{\frac{\pi_i \ln n}{n\gamma}} \right ) & \leq  2 \exp \left( - \frac{\gamma c_3^2 \frac{n \pi_i \ln n}{\gamma}}{4n \pi_i + 10 c_3 \sqrt{\frac{n\pi_i \ln n}{\gamma}}} \right) \\
& \leq 2 \exp \left( - \frac{c_3^2 n \pi_i \ln n}{4n\pi_i + 10c_3 n\pi_i} \right) \\
& = \frac{2}{n^\beta}.
\end{align*}
\end{enumerate}
\end{proof}

Now we are ready to prove~Lemma~\ref{lemma.goodeventshighprobability}.
We only consider $\calG_{\opt}$ and the upper bound on $P(\calG_{\emp})$ follows from the same steps.
By the union bound,
it suffices to upper bound the probability of the complement of each event in the definition of the ``good'' event $\calG_{\opt}$ (cf.~Definition \ref{def.goodevents}).

For the first part of the definition, the probability of ``bad'' events $\calE_i^c$ in \prettyref{eq:Ei} are upper bounded by
\begin{align}
\sum_{i\in [S]} P(\calE_i^c) \leq S \cdot \frac{2}{n^\beta},
\end{align}
where $\beta = \frac{c_3^2}{4 + 10c_3}$ as in Lemma~\ref{lemma.confidencepi}. Since we have assumed that $c_3 \geq 20$, we have $\beta \geq 1$.

For the second part of the definition,
applying Lemma \ref{lemma.concentrationentropy}, the overall probability of ``bad'' events $\calH_i^c$ in \eqref{eqn.goodeventiidentropy} are upper bounded by
\begin{align*}
& \sum_{i\in [S]} P(\calH_i^c) \mathbbm{1}\left( \pi_i \geq n^{c_4-1} \vee 100c_3^2 \frac{\ln n}{n\gamma} \right) \\
& \leq \sum_{i = 1}^S 2 \cdot c_3 \sqrt{\frac{n\pi_i \ln n}{\gamma}} \frac{2}{ \left (n\pi_i - c_3 \sqrt{\frac{n\pi_i \ln n}{\gamma}}\right )^\beta  } \mathbbm{1}\left( \pi_i \geq n^{c_4-1} \vee 100c_3^2 \frac{\ln n}{n\gamma} \right) \\
& \leq \sum_{i = 1}^S 2 \cdot c_3 \sqrt{\frac{n\pi_i \ln n}{\gamma}} \frac{2}{(9 n \pi_i /10)^\beta} \mathbbm{1}\left( \pi_i \geq n^{c_4-1} \vee 100c_3^2 \frac{\ln n}{n\gamma} \right) \\
& \leq \sum_{i = 1}^S 2c_3 n\pi_i\frac{2}{(9 n \pi_i /10)^\beta} \mathbbm{1}\left( \pi_i \geq n^{c_4-1} \vee 100c_3^2 \frac{\ln n}{n\gamma} \right) \\
& = \sum_{i = 1}^S \frac{4c_3 n\pi_i}{(9 n\pi_i /10)^\beta}  \mathbbm{1}\left( \pi_i \geq n^{c_4-1} \vee 100c_3^2 \frac{\ln n}{n\gamma} \right) \\
& \leq D  \frac{S}{n^{c_4 (\beta -1)}},
\end{align*}
where $D \triangleq \frac{4c_3 (10)^\beta}{9^\beta}$ and the second step follows from the fact that $\pi_i \mapsto n\pi_i - c_3 \sqrt{\frac{n\pi_i \ln n}{\gamma}}$ is increasing when $\pi_i \geq 100c_3^2 \frac{\ln n}{n\gamma}$.

\subsection{Proof of Lemma~\ref{lemma.reductiontoim}}

We simulate a Markov chain sample path with transition matrix $T_{ij}$ and stationary distribution $\pi_i$ from the independent multinomial model as described in Definition~\ref{def.immodel}, and define the estimator $\hat{H}_2$ as follows: output zero if the event $\cap_{1\leq i\leq S} \mathcal{E}_i$ does not happen (where $\mathcal{E}_i$ are events defined in Definition~\ref{def.goodevents}); otherwise, we set
\begin{align*}
\hat{H}_2(X_0,(W_{ij})_{i\in [S], j\le m_i}) = \hat{H}_1(X_0,(W_{ij})_{i\in [S], j\le n_i}).
\end{align*}
Note that this is a valid definition since $\cap_{1\leq i\leq S} \mathcal{E}_i$ implies $n_i\le m_i$ for any $i\in [S]$. As a result,
\begin{align}
P_{\mathsf{IM}}\left( |\hat{H}_2 - \bar{H}|\geq \epsilon \right) & \leq P_{\mathsf{IM}}\left( \left( \cap_{1\leq i\leq S} \mathcal{E}_i \right)^c \right) + P_{\mathsf{IM}} \left( \cap_{1\leq i\leq S} \mathcal{E}_i \right) P_{\mathsf{IM}} \left( |\hat{H}_2 - \bar{H}| \geq \epsilon | \cap_{1\leq i\leq S} \mathcal{E}_i \right).
\end{align}

It follows from Lemma~\ref{lemma.goodeventshighprobability} that
\begin{align}
P_{\mathsf{IM}}\left( \left( \cap_{1\leq i\leq S} \mathcal{E}_i \right)^c \right) & \leq \frac{2S}{n^\beta},
\end{align}
where $\beta = \frac{c_3^2}{4 + 10c_3}\geq 1$. Now, it suffices to upper bound $P_{\mathsf{IM}} \left( |\hat{H}_2 - \bar{H}| \geq \epsilon | \cap_{1\leq i\leq S} \mathcal{E}_i \right)$. The crucial observation is that the joint distribution of $(X_0, (n_i)_{i\in [S]}, (W_{ij})_{i\in [S],j\le n_i})$ are identical in two models, and thus
\begin{align}
P_{\mathsf{IM}} \left( |\hat{H}_2 - \bar{H}| \geq \epsilon | \cap_{1\leq i\leq S} \mathcal{E}_i \right) & = P_{\mathsf{MC}} \left( |\hat{H}_1 - \bar{H}| \geq \epsilon | \cap_{1\leq i\leq S} \mathcal{E}_i \right) \\
P_{\mathsf{IM}} \left( \cap_{1\leq i\leq S} \mathcal{E}_i \right) & = P_{\mathsf{MC}} \left( \cap_{1\leq i\leq S} \mathcal{E}_i \right).
\end{align}

By definition, the estimator $\hat{H}_1$ satisfies
\begin{align}
P_{\mathsf{MC}}\left( \cap_{1\leq i\leq S} \mathcal{E}_i,  |\hat{H}_1 - \bar{H}| \geq \epsilon  \right) & \leq \delta.
\end{align}
A combination of the previous inequalities gives
\begin{align}
P_{\mathsf{IM}}\left( |\hat{H}_2 - \bar{H}|\geq \epsilon \right)  & \leq \frac{2S}{n^\beta} + \delta.
\end{align}
as desired.

\subsection{Proof of Lemma~\ref{lemma.poissontomultinomial}}



We can simulate the independent multinomial model from the independent Poisson model by conditioning on the row sum. For each $i$, conditioned on $M_i\triangleq \sum_{j = 1}^S C_{ij} = m_i$, the random vector $C_i=(C_{i1},C_{i2},\ldots, C_{iS})$ follows the multinomial distribution $\mathsf{multi}\left( m_i, T_i \right)$, where $T=T(R)$ is the transition matrix obtained from normalizing $R$. In particular,
$T_i = \frac{1}{\sum_{j = 1}^S R_{ij}} ( R_{i1},R_{i2},\ldots,R_{iS})$.
Furthermore, $C_1,\ldots,C_S$ are conditionally independent. Thus, to apply the estimator $\hat{H}_1$ designed for the independent multinomial model with parameter $n$ that fulfills the guarantee \prettyref{eq:imguarantee1}, we need to guarantee that
\begin{align}\label{eqn.immodelrowcountrequirement}
M_i & > n\pi_i + c_3 \max \left \{ \frac{\ln n}{\gamma^*}, \sqrt{\frac{n\pi_i \ln n}{\gamma^*}} \right\},
\end{align}
for all $i$ with probability at least $1 - S n^{-c_3/2}$. Here $c_3\geq 20$ is the constant in Definition~\ref{def.immodel}, and
\begin{align}
\pi_i & = \frac{\sum_{j =1}^S R_{ij} }{r},
\end{align}
where $r = \sum_{1\leq i,j\leq S} R_{ij}$.
Note that $M_i \sim \Poi(\lambda_i)$, where
$\lambda_i \triangleq \frac{4 n}{\tau} \sum_j R_{ij} = \frac{4 n r }{\tau} \pi_i \geq 4 n \pi_i$, due to the assumption that $r \geq \tau$.
By the assumption of $\pi_i \geq \pi_{\min} \geq \frac{c_3 \ln n}{n \gamma}$, we have
\begin{equation}
n \pi_i \geq c_3 \max \left \{ \frac{\ln n}{\gamma^*}, \sqrt{\frac{n\pi_i \ln n}{\gamma^*}} \right\}.
\label{eq:verify}
\end{equation}
Then
\begin{align}
\prob{M_i < n\pi_i + c_3 \max \left \{ \frac{\ln n}{\gamma^*}, \sqrt{\frac{n\pi_i \ln n}{\gamma^*}} \right\}}
\leq & ~  \prob{\Poi(4 n \pi_i) < 2 n \pi_i} 	\nonumber \\
\stepa{\leq} & ~ \exp(-n \pi_i/2)	\\
\stepb{\leq} & ~ n^{-c_3/2},
\end{align}
where
(a) follows from Lemma~\ref{lemma.poissontail}; (b) follows from $\pi_i \geq \pi_{\min} \geq \frac{ c_3 \ln n}{n \gamma} \geq \frac{ c_3 \ln n}{n}$. This completes the proof.

\subsection{Proof of Lemma~\ref{lemma.indistinguishableipmodel}}

The dependence diagram for all random variables is as follows:
\[
\begin{tikzcd}
\bfU \arrow{r} 	& \bfR \arrow{r} \arrow{d} & \pi(\bfR) \arrow{r} 		& X_0\\
												& \bfC &  &
\end{tikzcd}
\]
where $\pi(\bfR)=(\pi_0(\bfR),\pi_1(\bfR),\cdots,\pi_S(\bfR))$ is the stationary distribution defined in \prettyref{eq:piR} obtained by normalizing the matrix $\bfR$.
Recall that for $i=1,2$, $F_i$ denotes the joint distribution on the sufficient statistic $(X_0,\bfC)$ under the prior $\mu_i$.
Our goal is to show that $\TV(F_1,F_2)\to 0$.
Note that $X_0$ and $\bfC$ are dependent; however, the key observation is that, by concentration, the distribution of $X_0$ is close to a fixed distribution $P_0$ on the state space $\{0,1,\cdots,S\}$, where $P_0\triangleq \frac{}{S(1+\sqrt{\alpha S})} (S\sqrt{\alpha S},1,1,\cdots,1)$.
Thus, $X_0$ and $\bfC$ are approximately independent. For clarity, we denote $F_1=P_{X_0,\bfC}, F_2=Q_{X_0,\bfC}$. By the triangle inequality of the total variation distance, we have
\begin{align}
\mathsf{TV}(F_1,F_2) \le \mathsf{TV}(P_{X_0,\bfC}, P_0\otimes P_{\bfC}) + \mathsf{TV}(P_0\otimes P_{\bfC}, P_0\otimes Q_{\bfC}) + \mathsf{TV}(Q_{X_0,\bfC}, P_0\otimes Q_{\bfC}). \label{eq.TV_triangle}
\end{align}

To upper bound the first term, note that $\bfC \to \bfR \to X_0$ forms a Markov chain. Hence, by the convexity of total variation distance, we have
\begin{align}
\mathsf{TV}(P_{X_0,\bfC}, P_0\otimes P_{\bfC}) &= \bE_{P_{\bfC}}[ \mathsf{TV}(P_{X_0|\bfC}, P_0)] \\
&\le \bE_{P_{\bfR}}[\mathsf{TV}(P_{X_0|\bfR}, P_0)] \nonumber \\
&= \bE [\mathsf{TV}(\pi(\bfR), P_0)] \nonumber \\
&= \frac{1}{2} \left(\bE \left|\pi_0(\bfR)-\frac{1}{1+\sqrt{\alpha S}}\right| + \sum_{i=1}^S \bE \left|\pi_i(\bfR)-\frac{1}{S(1+\sqrt{\alpha S})}\right|\right). \label{eq.TV_first_term}
\end{align}
 We start by showing that the row sums of $\bfR$ concentrate. Let $r_i = \sum_{i=0}^S R_{ij} = a + \sum_{i=1}^S U_{ij}$, where $a=\sqrt{\alpha S}$.
It follows from the Hoeffding inequality in Lemma~\ref{lem_hoeffding} that
\begin{align}
\prob{ \left| r_i - (\sqrt{\alpha S} + \alpha S) \right| \geq u,~~i=1,\ldots,S } & \leq 2S \exp \left( \frac{-2u^2}{S \left( \frac{d_1^2 S \ln S}{n} \right)^2} \right) \to0,
\label{eq:ri}
\end{align}
provided that $u\gg \frac{(S\ln S)^{3/2}}{n}$.

Next consider the entrywise sum of $\bfR$.
Write $r \triangleq \sum_{0\leq i,j\leq S} R_{ij} =b+2aS +  \sum_{1\leq i<j\leq S} 2U_{ij} + \sum_{1\leq i\leq S} U_{ii}$.
Note that $\Expect[r]= b+2aS + S^2 \alpha = S (1+\sqrt{\alpha S})^2 $, by \prettyref{eq:ab}.
Then, it follows from the Hoeffding inequality in Lemma~\ref{lem_hoeffding} that
\begin{align}
\prob{ \left| r  - S(1+\sqrt{\alpha S})^2 \right| \geq \sqrt{S}u} & \leq 2\exp\left( \frac{-2Su^2}{\frac{S(S-1)}{2} 4 \left( \frac{d_1^2 S\ln S}{n} \right)^2 + S \left( \frac{d_1^2 S\ln S}{n} \right)^2}\right) \to 0
\label{eq:r}
\end{align}
provided that $u \gg \frac{S^{3/2} \ln S}{n}$. Henceforth, we set
\begin{equation}
u = \frac{(S\ln^2 S)^{3/2}}{n}.
\label{eq:u}
\end{equation}

Hence, with probability tending to one, $|r_i-(\sqrt{\alpha S}+\alpha S)|\le u$ for $i=1,2,\cdots,S$ and $|r-S(1+\sqrt{\alpha S})^2|\le \sqrt{S}u$. Conditioning on this event, for $i=1,2,\cdots,S$ we have
\begin{align}
\left|\pi_i(\bfR) - \frac{\sqrt{\alpha S}}{S(1+\sqrt{\alpha S})}\right| \le (\sqrt{\alpha S}+\alpha S)\left|\frac{1}{r}-\frac{1}{\bE[r]}\right| + \frac{u}{r} \le \frac{2u}{S^\frac{3}{2}} + \frac{u}{S}.
\end{align}
For $i=0$, $\pi_0(\bfR)=\frac{S(1+\sqrt{\alpha S})}{r}$, we have
\begin{align}
\left|\pi_0(\bfR) - \frac{1}{1+\sqrt{\alpha S}}\right| = S(1+\sqrt{\alpha S})\left|\frac{1}{r}-\frac{1}{\bE[r]}\right| \le \frac{2u}{\sqrt{S}}.
\end{align}
Therefore, in view of \eqref{eq.TV_first_term}, we have
\begin{align}
\mathsf{TV}(P_{X_0,\bfC}, P_0\otimes P_{\bfC})
\le \left(\frac{2u}{\sqrt{S}}+\sum_{i=1}^S \left(\frac{2u}{S^\frac{3}{2}} + \frac{u}{S}\right)\right) + 2\cdot o(1) = \frac{4u}{\sqrt{S}} + u + o(1) = o(1)
\end{align}
as $S\to\infty$. Similarly, we also have $\mathsf{TV}(Q_{X_0,\bfC}, P_0\otimes Q_{\bfC})=o(1)$.

By \eqref{eq.TV_triangle}, it remains to show that $\mathsf{TV}(P_0\otimes P_{\bfC},P_0\otimes Q_{\bfC})=o(1)$. Note that $P_\bfC, Q_\bfC$ are products of Poisson mixtures, by the triangle inequality of total variation distance again we have
\begin{align}
\mathsf{TV}(P_0\otimes P_{\bfC},P_0\otimes Q_{\bfC}) = \mathsf{TV}(P_\bfC,Q_\bfC) \le \sum_{1\le i\le j\le S}\mathsf{TV}\left(\bE[\spo(\frac{4n}{\tau} U_{ij})], \bE[\spo(\frac{4n}{\tau} U_{ij}')]\right). \label{eq.TV_second_term}
\end{align}
We upper bound the individual terms in \eqref{eq.TV_second_term}. For the total variation distance between Poisson mixtures, note that the random variables $\frac{4n}{S} U_{ij}$ and $\frac{4n}{S} U'_{ij}$ match moments up to order
\begin{align}
\frac{D}{\sqrt{\eta}} & = D d_1 \ln S,
\end{align}
and are both supported on $[0, \alpha \eta^{-1} \cdot \frac{4n}{S}] = [0, \frac{d_1^2 S \ln S}{n} \frac{4n}{S}] = [0, 4d_1^2 \ln S]$. It follows from Lemma~\ref{lemma.poissontv} that if
\begin{align}
D d_1 \ln S \geq 8e^2 d_1^2 \ln S,
\end{align}
we have
\begin{align}
\mathsf{TV}\left( \mathbb{E}\left[ \mathsf{Poi}\left( \frac{4n}{S} U_{ij} \right) \right], \mathbb{E}\left[ \mathsf{Poi}\left( \frac{4n}{S} U'_{ij} \right) \right]  \right) & \leq \frac{1}{2^{D d_1 \ln S}} \\
& \leq \frac{1}{S^{\frac{D^2 \ln 2}{4e^2}}} \\
&  \leq  \frac{1}{S^{100}}.
\end{align}
where we set $d_1 = \frac{D}{8e^2}$ and used the fact that $D\geq 100$. By \eqref{eq.TV_second_term},
\begin{align}
\mathsf{TV}(P_0\otimes P_{\bfC},P_0\otimes Q_{\bfC}) \le S^2\cdot\frac{1}{S^{100}} = \frac{1}{S^{98}} = o(1)
\end{align}
as $S\to\infty$, establishing the desired lemma.


\subsection{Proof of Lemma~\ref{lemma.functionalseperation}}
Let $\Delta = \frac{c S^2}{8 n \log S} = \frac{c \alpha S}{8}$, where $c$ is the constant from Lemma \ref{lemma.wuyangprior}.
Recall that $\phi(x) = x \log \frac{1}{x}$. In view of \prettyref{eq:HTR2}, we have
\begin{align*}
\bar H(T(\bfR))
= & ~ \frac{1}{r} \left( \sum_{i,j=0}^{S} \phi(R_{ij}) - \sum_{i = 0}^S \phi(r_i) \right) \\
= & ~ \frac{1}{r} (\phi(b) + 2 S \phi(a) + \phi(b+a S)) + \frac{1}{r}
\sum_{i,j=1}^{S} \phi(R_{ij})  -  \frac{1}{r} \sum_{i = 0}^S \phi(r_i)  \\
= & ~
\underbrace{\frac{1}{r} (\phi(b) + 2 S \phi(a) + \phi(b+a S))}_{H_1} + \underbrace{\frac{1}{r}
	\pth{2 \sum_{1\leq i<j \leq S}  \phi(U_{ij}) + \sum_{1\leq i\leq S} \phi(U_{ii})}}_{H_2} -  \underbrace{\frac{1}{r} \sum_{i = 0}^S \phi(r_i)}_{H_3},
\end{align*}
where the last step follows from the symmetry of the matrix $\bfU$.

For the first term, note that $|\phi(b) + 2 S \phi(a) + \phi(b+a S)| = |\phi(S) + 2 S \phi(\sqrt{\alpha S}) + \phi(S(1+\sqrt{\alpha S}))| \leq 10 S \log S$.
Thus, conditioned on \prettyref{eq:r}, we have
\begin{equation}
\left|\frac{1}{r} - \frac{1}{\Expect[r]}\right| \leq \frac{u}{S^{\frac{3}{2}}},
\label{eq:rr}
\end{equation}
where $\Expect[r]=S(1+\sqrt{\alpha S})^2$.
Put $h_1 \triangleq \frac{\phi(b) + 2 S \phi(a) + \phi(b+a S)}{S(1+\sqrt{\alpha S})^2}$, we have
\begin{equation}
|H_1 - h_1| \leq \frac{10u \ln S}{\sqrt{S}}.
\label{eq:h1}
\end{equation}
with probability tending to one.

For the second term, by Definition~\ref{con.priorconstruction}, for any $i,j$, $U_{ij}$ is supported on $[0,\frac{d_1^2 S \ln S}{n}]$. Thus, $\phi(U_{ij})$ is supported on $[0,\frac{d_1^2 S\ln S}{n}\ln \frac{n}{d_1^2 S \ln S}]$ for any $i,j$.
Hence, it follows from the Hoeffding inequality in Lemma~\ref{lem_hoeffding} that
\begin{align}
& \prob{\left | 2 \sum_{1\leq i<j \leq S}  \phi(U_{ij}) + \sum_{1\leq i\leq S} \phi(U_{ii}) - S^2 \Expect[\phi(U)]
	\right | \geq \frac{\Delta S}{4}} \\
& \leq 2 \exp\left( \frac{-2(\Delta S/4)^2}{\sum_{1\leq i<j\leq S} \left( 2 \frac{d_1^2 S\ln S}{n}\ln \frac{n}{d_1^2 S \ln S} \right)^2 + \sum_{1\leq i\leq S} \left( \frac{d_1^2 S \ln S}{n}\ln \frac{n}{d_1^2 S \ln S}\right)^2}  \right) \\
& \leq 2 \exp\left( - \Omega \left( \frac{S^2}{(\ln n)^2(\ln S)^4} \right) \right) \\
& \to 0
\end{align}
as $S\to \infty$, provided that $\ln n \ll \frac{S}{\ln^2 S}$.
Put $h_2 = \frac{S}{(1+\sqrt{\alpha S})^2} \Expect[\phi(U)] $. Using \prettyref{eq:rr} and the fact that $0 \leq \phi(U) \leq \frac{d_1^2 S\ln S}{n}\ln \frac{n}{d_1^2 S \ln S}$, we have
\begin{equation}
|H_2 - h_2| \leq u\cdot\frac{d_1^2 S^{\frac{3}{2}}\ln S}{n}\ln \frac{n}{d_1^2 S \ln S} + \frac{\Delta}{4}.
\label{eq:h2}
\end{equation}
For the third term, condition on the event in \prettyref{eq:ri}, we have
$|\phi(r_i) - \phi(\alpha S + \sqrt{\alpha S})| \leq C u \ln S$  and $|\phi(r_i)|\le C$,
for some absolute constant $C$.
Put $h_3 = \frac{1}{(1+\sqrt{\alpha S})^2} \phi(\alpha S + \sqrt{\alpha S})$. We have
\begin{equation}
|H_3 - h_3| \leq \frac{C u}{\sqrt{S}} + C u \ln S .
\label{eq:h3}
\end{equation}

Finally, combining \prettyref{eq:h1}, \prettyref{eq:h2}, \prettyref{eq:h3} as well as \prettyref{eq:u},
with probability tending to one,
\begin{equation}
|\bar H(T(\bfR))  - (h_1+h_2-h_3)| \leq
\frac{\Delta}{4} + C' \frac{S^{3/2} \ln^4 S }{n}
\end{equation}
for some absolute constant $C'$. Likewise, with probability tending to one, we have
\begin{equation}
|\bar H(T(\bfR'))  - (h_1+h_2'-h_3)| \leq
\frac{\Delta}{4} + C' \frac{S^{3/2} \ln^4 S }{n}
\end{equation}
where $h_2' = \frac{S}{(1+\sqrt{\alpha S})^2} \Expect[\phi(U')] $.
In view of Lemma~\ref{lemma.wuyangprior}, we have
\begin{align}
|h_2-h_2'| \geq  \frac{c \alpha S}{(1+\sqrt{\alpha S})^2} \geq 2\Delta.
\end{align}
This completes the proof.

\subsection{Proof of Lemma~\ref{lemma.spectralgapcontrol}}

We only consider the random matrix $\bfR=(R_{ij})$ which is distributed according to the prior $\mu_1$; the case of $\mu_2$ is entirely analogous.

First we lower bound $\pi_{\min}$ with high probability. Recall the definition of $u$ in \eqref{eq:u}, and \eqref{eq:ri}, \eqref{eq:r}. Since $u \leq \alpha S = \frac{S^2}{n \ln S}$, we have $S\le r\le 5S$ and $r_i \geq \sqrt{\alpha S}$ for all $i\in [S]$ with probability tending to one.
Furthermore, $r_0 = S(1+\sqrt{\alpha S})\ge \sqrt{\alpha S}$.
Consequently,
\[
\pi_{\min} = \frac{1}{r} \min_{0\leq i\leq S} r_i \geq \frac{\sqrt{\alpha S}}{5S} = \frac{1}{5 \sqrt{n \ln S}},
\]
as desired.

Next, we deal with the spectral gap. Recall $T=T(R)$ is the normalized version of $R$. 
Let $D = \diag(r_0,\ldots,r_S)$ and
$D_\pi = \diag(\pi_0,\ldots,\pi_S)$, where $r_i = \sum_{j=0}^S R_{ij}$, $r = \sum_{i,j=0}^S R_{ij}$, and $\pi_i = \frac{r_i}{r}$.
Then we have $T=D^{-1} R$. Furthermore, by the reversiblity of $T$,
\begin{equation}
T'\triangleq  D_{\pi}^{1/2} T D_{\pi}^{-1/2} = D^{-1/2} R D^{-1/2},
\label{eq:Tp}
\end{equation}
is a symmetric matrix. Since $T'$ is a similarity transform of $T$, they share the same spectrum. Let $1= \lambda_1(T) \geq \ldots \geq \lambda_{S+1}(T)$ (recall that $T$ is an $(S+1)\times (S+1)$ matrix). In view of \prettyref{eq:R}, we have
\[
L \triangleq  \Expect[\bfR] = \left[
\begin{array}{c|c}
  b & a \cdots a \\ \hline
  a & \alpha \cdots \alpha \\
  \vdots & \vdots \\
  a & \alpha \cdots \alpha
\end{array}
\right],
\quad
Z \triangleq \bfR - \Expect[\bfR]=
\left[
\begin{array}{c|c}
  0 & 0 \cdots 0 \\ \hline
  0 & \raisebox{-15pt}{\mbox{{$\bfU-\Expect[\bfU]$}}} \\[-4ex]
  \vdots & \\[-0.5ex]
  0 &
\end{array}
\right]
\]
Crucially, the choice of $a=\sqrt{\alpha S},b=S$ in \prettyref{eq:ab} is such that $b \alpha = a^2$, so that $\Expect[\bfR]$ is a symmetric positive semidefinite rank-one matrix. Thus, we have from \prettyref{eq:Tp}
\[
T' = D^{-1/2} L D^{-1/2} + D^{-1/2} Z D^{-1/2}.
\]
Note that $L' \triangleq D^{-1/2} L D^{-1/2}$ is also a symmetric positive semidefinite rank-one matrix.
Let $\lambda_1(L') \geq 0 = \lambda_2(L') = \cdots =\lambda_{S+1}(L')$.
By Weyl's inequality \cite[Eq.~(1.64)]{tao2012topics}, for $i=2,\ldots,S+1$, we have
\begin{equation}
|\lambda_i(T)| \leq \|D^{-1/2} Z D^{-1/2}\|_2 \leq \|D^{-1/2}\|_2^2 \|Z\|_2 = \frac{1}{\min_{0\leq i \leq S} r_i } \|\bfU - \Expect[\bfU]\|_2.
\label{eq:weyl}
\end{equation}
Here and below $\|\cdot\|_2$ stands for the spectral norm (largest singular values).
So far everything has been determinimistic. Next we show that with high probability, the RHS of \prettyref{eq:weyl} is at most $\Omega(\sqrt{\frac{S \ln^3 S}{n}} )$.

Note that $\bfU - \Expect[\bfU]$ is a zero-mean Wigner matrix. Furthermore, $U_{ij}$ takes values in $[0,\alpha \eta^{-1}] = [0,\frac{d_1^2 S \ln S}{n}]$, where $d_1$ is an absolute constant. It follows from the standard tail estimate of the spectral norm for the Wigner ensemble (see, e.g.~\cite[Corollary 2.3.6]{tao2012topics}) that there exist universal constants $C,c,c'>0$ such that
\begin{align}
P \left( \|\bfU - \Expect[\bfU]\|_2 > \frac{C S^{3/2} \ln S}{n}  \right) & \leq c' e^{-c S}.
\label{eq:wigner}
\end{align}
Combining \prettyref{eq:ri}, \prettyref{eq:weyl}, and \prettyref{eq:wigner}, the absolute spectral gap of $T=T(\bfR)$ satisfies
\[
\prob{\gamma^*(T(\bfR)) \geq 1 - C \sqrt{\frac{S \ln^3 S}{n}} } \to 1,
\]
as $S\to\infty$. By union bound, we have shown that $\prob{\bfR \in \calR\left( S, \gamma,\tau,q\right)} \to 1$, with $\gamma,\tau,q$ as chosen in Lemma~\ref{lemma.spectralgapcontrol}.

\subsection{Proof of Lemma~\ref{lemma.representationofempiricalentropy}}
The representation \prettyref{eq:Hemp2} follows from definition of conditional entropy. It remains to show \prettyref{eq:Hemp1}.
%
Let $\hat P$ denote the transition matrix corresponding to the empirical conditional distribution, that is,
$\hat P_{ij} \triangleq \hat P_{X_2=j|X_1=i}$. Then, for any transition matrix $P=(P_{ij})$,
\begin{align*}
\frac{1}{n} \ln \frac{1}{P_{X_1^n|X_0}(x_1^n|x_0)} & = \frac{1}{n} \sum_{m = 1}^n \sum_{i = 1}^S \sum_{j = 1}^S \mathbbm{1}(x_{m-1}=i, x_m = j) \ln \frac{1}{P_{ij}} \\
& = \frac{1}{n} \sum_{i = 1}^S \sum_{j = 1}^S (n\hat{\pi}_i \hat{P}_{ij}) \ln \frac{1}{P_{ij}} \\
& = \sum_{i = 1}^S \hat{\pi}_i \sum_{j = 1}^S \hat{P}_{ij} \ln \frac{1}{P_{ij}} \\
& = \Hemp + \sum_{i = 1}^S \hat{\pi}_i D(\hat{P}_{i\cdot} \| P_{i\cdot}),
\end{align*}
where in the last step $D(p\|q) = \sum_{i} p_i \ln \frac{p_i}{q_i}\geq 0$ stands for the Kullback--Leibler (KL) divergence between probability vectors $p$ and $q$.
Then \prettyref{eq:Hemp1} follows from the fact that the nonnegativity of the KL divergence.
%